\documentclass[letterpaper]{article} % DO NOT CHANGE THIS
\usepackage{aaai25}  % DO NOT CHANGE THIS
\usepackage{times}  % DO NOT CHANGE THIS
\usepackage{helvet}  % DO NOT CHANGE THIS
\usepackage{courier}  % DO NOT CHANGE THIS
\usepackage[hyphens]{url}  % DO NOT CHANGE THIS
\usepackage{graphicx} % DO NOT CHANGE THIS
\usepackage{natbib}  % DO NOT CHANGE THIS AND DO NOT ADD ANY OPTIONS TO IT
\usepackage{caption} % DO NOT CHANGE THIS AND DO NOT ADD ANY OPTIONS TO IT
\usepackage{algorithm}
\usepackage{algorithmic}
\usepackage{amsmath}
\usepackage{amssymb}
\usepackage{mathtools}
\usepackage{amsthm}
\usepackage{bm}
\usepackage{color}

\newtheorem{dfn}{Definition}
\newtheorem{thm}{Theorem}

\newtheorem{lem}{Lemma}

\newtheorem{rem}{Remark}

\newtheorem{assum}{Assumption}

\usepackage{newfloat}
\usepackage{listings}
\DeclareCaptionStyle{ruled}{labelfont=normalfont,labelsep=colon,strut=off} % DO NOT CHANGE THIS
\floatstyle{ruled}
\newfloat{listing}{tb}{lst}{}
\floatname{listing}{Listing}
\title{Explicit and Implicit Graduated Optimization in Deep Neural Networks}
\author{
    Naoki Sato,
    Hideaki Iiduka
}
\affiliations{
    Meiji University\\
    naoki310303@gmail.com, iiduka@cs.meiji.ac.jp
}

\usepackage{bibentry}
\begin{document}

\maketitle

\begin{abstract}
Graduated optimization is a global optimization technique that is used to minimize a multimodal nonconvex function by smoothing the objective function with noise and gradually refining the solution. This paper experimentally evaluates the performance of the explicit graduated optimization algorithm with an optimal noise scheduling derived from a previous study and discusses its limitations. It uses traditional benchmark functions and empirical loss functions for modern neural network architectures for evaluating. In addition, this paper extends the implicit graduated optimization algorithm, which is based on the fact that stochastic noise in the optimization process of SGD implicitly smooths the objective function, to SGD with momentum, analyzes its convergence, and demonstrates its effectiveness through experiments on image classification tasks with ResNet architectures.
\end{abstract}

\begin{links}
     \link{Code}{https://github.com/iiduka-researches/igo-aaai25}
     %\link{Extended version}{https://aaai.org/example/extended-version}
\end{links}

\section{Introduction}
Nonconvex optimization problems are ubiquitous in the machine-learning and artificial-intelligence fields, and those arising in training deep neural networks (DNN) \cite{Bengio2009Lea} are particularly interesting from multiple aspects. In particular, because the nonconvex functions that appear in DNN training have many local optimal solutions, stochastic gradient descent (SGD) \cite{Robbins1951Ast} and its variants, which use gradients to update the sequence, may fall into local optimal solutions and may not minimize losses sufficiently.

Graduated optimization \cite{Andrew1987Vis} is a global optimization method that can be applied to multimodal functions in order to avoid local optimal solutions and converge to the global optimal one. The method first prepares $M$ monotone decreasing noises $(\delta_m)_{m \in [M]}$. Then, by smoothing the original objective function $f$ with that noise, a sequence of $M$ smoothed functions $(\hat{f}_{\delta_m})_{m \in [M]}$ is obtained that gradually approaches the original objective function. The function $\hat{f}_{\delta_1}$ smoothed with the largest noise $\delta_1$ is then optimized, and the approximate solution is then used as the initial point for optimization of the function $\hat{f}_{\delta_2}$ smoothed with the second largest noise $\delta_2$. After that, the function $\hat{f}_{\delta_3}$ smoothed with the third largest noise $\delta_3$ is optimized with the second approximate solution as the initial point, and the procedure is repeated in order to search for the global optimal solution without falling into a local optimal one. See Figure \ref{fig:00} for a conceptual diagram of this process.

Graduated optimization is often used in machine learning and computer vision for, e.g., semi-supervised learning \cite{Olivier2006Aco, Vikas2006Det, Olivier2008Opt} and robust estimation \cite{Heng2020Gra, Antonante2022Out, Peng2023Ont}. In addition, score-based generative models \cite{Song2019Gen, Song2020Imp} and diffusion models \cite{Dickstein2015Dee, Ho2020Den, Song2021Sco, Robin2022Hig}, which are state-of-the-art models of image generation, use this approach.

Theoretical analysis of graduated optimization began with the pioneering work by \cite{Duchi2012Ran} on nonsmooth convex functions, and several papers \cite{Hossein2015ATh, Elad2016OnG, Iwakiri2022Sin, Li2023Sto, Sato2023Usi} have produced important results. In particular, \cite{Elad2016OnG} defined a family of nonconvex functions that satisfy the conditions for convergence of a graduated optimization algorithm, called $\sigma$-nice, and proposed a first-order algorithm based on graduated optimization. They also studied the convergence and convergence rate of the algorithm to a global optimal solution for $\sigma$-nice functions. In addition, \cite{Sato2023Usi} showed that the stochastic noise in the optimization process that SGD naturally has can implicitly smooth the objective function, and they proposed an implicit graduated optimization algorithm that exploits this property. Furthermore, they defined a new $\sigma$-nice function, which is an extension of the $\sigma$-nice function, and provided a convergence analysis of the implicit graduated optimization algorithm for this function. Here, this paper refers to graduated optimization with explicit smoothing operations (see Definition \ref{dfn:fhat}) as ``explicit graduated optimization'' and to graduated optimization with implicit smoothing operations as ``implicit graduated optimization''.

\begin{figure*}[htbp]
\begin{tabular}{c}
\begin{minipage}[t]{0.99\hsize}
\centering
\includegraphics[width=0.91\textwidth]{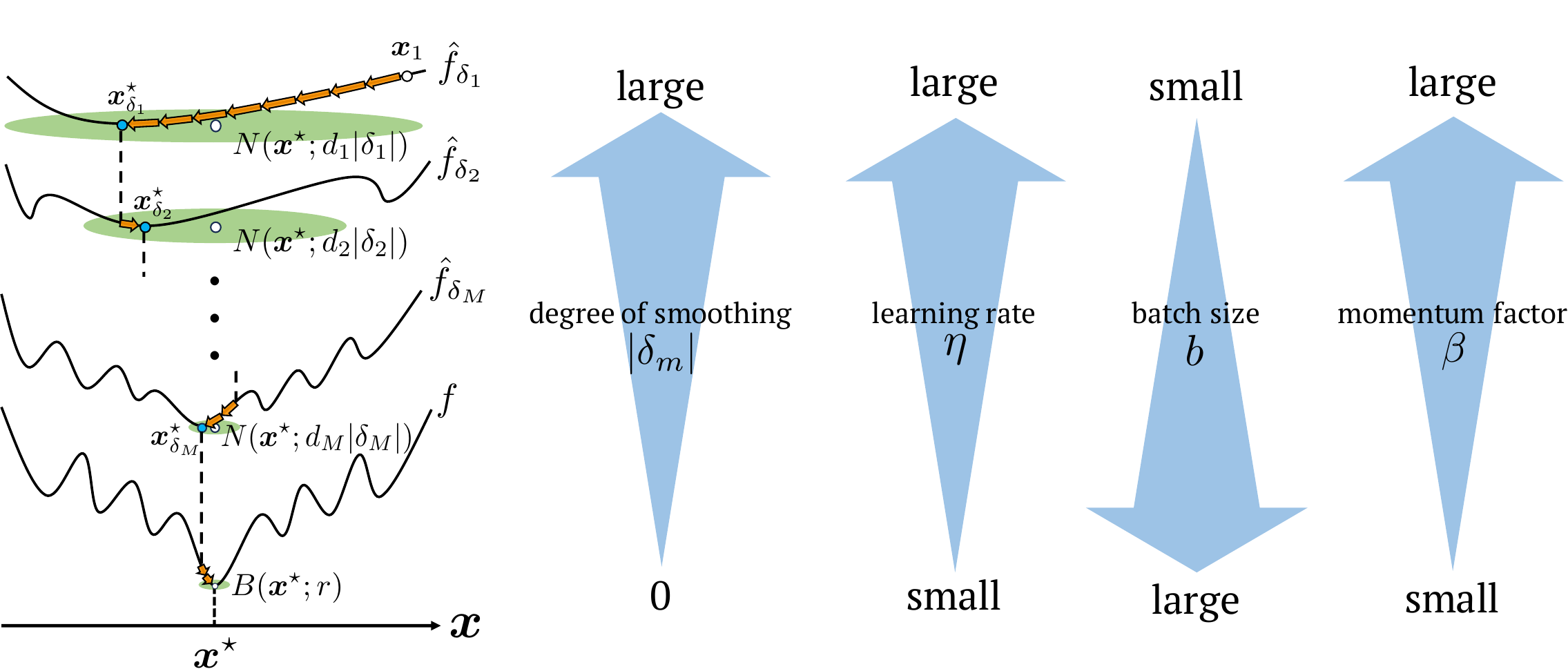}
\caption{Conceptual diagram of new $\sigma$-nice function and its smoothed versions and implicit graduated optimization. Note that this is a direct quotation from the prior work \cite{Sato2023Usi} with minor changes made to fit our results.}
\label{fig:00}
\end{minipage}
\end{tabular}
\vspace*{-12pt}
\end{figure*}

This paper analyzes the new $\sigma$-nice function and discusses the implicit graduated optimization algorithm, which has so far been studied in relation to stochastic noise in SGD, in relation to stochastic noise in SGD with momentum. It also complements the prior studies with important theoretical and experimental validations on explicit and implicit graduated optimization. 

\subsection{Contributions}
\textbf{1. An example of a new $\sigma$-nice function.} \cite{Sato2023Usi} proposed new $\sigma$-nice functions, i.e., a family of nonconvex functions that satisfy the necessary conditions for analysis and convergence of graduated optimization algorithms, but did not provide an example. Here, we theoretically prove that Rastrigin's function \cite{Torn1989Glo, Rudolph1990Glo}, one of the traditional benchmark functions for global optimization, is a new $\sigma$-nice function (Theorem \ref{thm:01}).

\noindent
\textbf{2. Experimental validation in explicit graduated optimization.} \cite{Sato2023Usi} theoretically derived the optimal decay rate of the noise in both explicit and implicit graduated optimization, but did not provide numerical experiments to determine whether it is valid for explicit graduated optimization (Algorithm \ref{alg:gnc}). Here, we applied the explicit graduated optimization algorithm and general global optimization methods to traditional benchmark functions for global optimization and found that the explicit graduated optimization algorithm with the optimal noise schedule outperforms the other methods on some problems (Table \ref{tab:1}). We also applied the explicit graduated optimization algorithm to DNNs and found that its performance is not superior to other methods (Figure \ref{fig:01}).

\noindent
\textbf{3. Experimental validation in implicit graduated optimization with SGD.} \cite{Sato2023Usi} proposed an implicit graduated optimization algorithm using stochastic noise in SGD (Algorithm \ref{alg:gnc2}), but did not provide numerical experimental results with the noise decay rate set as per theory. In this paper, we provide numerical experimental results that are consistent with theory as far as computational complexity permits, confirming that their algorithm works effectively (Figure \ref{fig:03}).

\noindent
\textbf{4. Theoretical development of implicit graduated optimization algorithm with SGD with momentum.} On the basis that the stochastic noise of SGD is determined by the ratio of learning rate and batch size and that this noise helps to the smooth the objective function, \cite{Sato2023Usi} proposed an algorithm that achieves implicit graduated optimization by varying those parameters during training so that the noise decreases gradually. In this paper, we extend this argument to SGD with momentum, and on the basis that the stochastic noise of SGD with momentum is determined by the learning rate, batch size, and momentum factor \cite{Sato2024Rol}, we constructed an algorithm (Algorithm \ref{alg:gnc3}) to perform implicit graduated optimization in the same way (see Figure \ref{fig:00}). We also analyzed  the algorithm's convergence and showed that it reaches an $\epsilon$-neighborhood of the global optimal solution of the new $\sigma$-nice function $f$ in $\mathcal{O}\left( 1/\epsilon^{\frac{1}{p}} \right)$ $(p \in (0,1])$ rounds (Theorem \ref{thm:02}).

\noindent
\textbf{5. Experimental validation of implicit graduated optimization with SGD with momentum.} We tested the effectiveness of the proposed algorithm by using it to train ResNet18 and WideResNet-28-10 on the CIFAR100 dataset for image classification tasks. We confirmed that the algorithm has a lower loss function value and higher test accuracy than those of vanilla SGD with momentum (Figure \ref{fig:04}). In addition, based on the theory of optimal noise decay rate in graduated optimization, the optimal decaying learning rate for SGD with momentum was tested by training ResNet34 on the ImageNet dataset. The results confirmed that the theoretically optimal learning rate scheduler, a polynomial decay scheduler with a power less than 1, achieves the lowest loss function value (Figure \ref{fig:05}).

\section{Preliminaries}
\textbf{Notation.} Let $\mathbb{N}$ denote the set of non-negative integers. For $m \in \mathbb{N} \backslash \{0\}$, define $[m] := \{1, 2, \ldots, m\}$. The space $\mathbb{R}^d$ is a $d$-dimensional Euclidean space with an inner product $\langle \cdot, \cdot \rangle$, which induces the norm $\| \cdot \|$. The Euclidean closed ball of radius $r$ centered at $\bm{\hat{x}}$ is denoted by $B(\bm{\hat{x}}; r) := \left\{ \bm{x} \in \mathbb{R}^d \colon \| \bm{x} - \bm{\hat{x}} \| \leq r \right\}$. The DNN is parameterized by a vector $\bm{x} \in \mathbb{R}^d$, which is optimized by minimizing the empirical loss function $f(\bm{x}) := \frac{1}{n}\sum_{i \in [n]} f_i(\bm{x})$, where $f_i(\bm{x})$ is the loss function for $\bm{x} \in \mathbb{R}^d$ and the sample $\bm{z}_i$ $(i \in [n])$. Let $\xi$ be a random variable independent of $\bm{x} \in \mathbb{R}^d$, and let $\mathbb{E}_{\xi}[X]$ denote the expectation with respect to $\xi$ of a random variable $X$. The random variable $\xi_{t,i}$ is generated from the $i$-th sampling at time $t$, and $\bm{\xi}_t := (\xi_{t,1}, \xi_{t,2}, \ldots, \xi_{t,b})$ is independent of the sequence $(\bm{x}_k)_{k=0}^{t} \subset \mathbb{R}^d$, where $b$ $(\leq n)$ represents the batch size. The independence of $\bm{\xi}_0, \bm{\xi}_1, \ldots$ allows us to define the total expectation $\mathbb{E}$ as $\mathbb{E} = \mathbb{E}_{\bm{\xi}_0}\mathbb{E}_{\bm{\xi}_1}\cdots\mathbb{E}_{\bm{\xi}_t}$. Let $\mathsf{G}_{\bm{\xi}_t}(\bm{x})$ be the stochastic gradient of $f(\cdot)$ at $\bm{x} \in \mathbb{R}^d$. The mini-batch $\mathcal{S}_t$ consists of $b$ samples $\bm{z}_i$ at time $t$, and the mini-batch stochastic gradient of $f(\bm{x}_t)$ for $\mathcal{S}_t$ is defined as $\nabla f_{\mathcal{S}_t}(\bm{x}_t) := \frac{1}{b}\sum_{i \in [b]} \mathsf{G}_{\xi_{t, i}} (\bm{x}_t)$.

\begin{dfn}
[Smoothed function] \label{dfn:fhat} Given an $L_f$-Lipschitz function $f$, define $\hat{f}_\delta$ to be the function obtained by smoothing $f$ as
\begin{align}\label{eq:06}
\hat{f}_\delta(\bm{x}) := \mathbb{E}_{\bm{u} \sim B(\bm{0}; 1)} \left[f(\bm{x} - \delta \bm{u}) \right],
\end{align}
where $\delta \in \mathbb{R}$ represents the degree of smoothing and $\bm{u}$ is a random variable distributed uniformly over $B(\bm{0}; 1)$. Also, 
\begin{align*}
\bm{x}^\star := \underset{\bm{x} \in \mathbb{R}^d} {\operatorname{argmin}} f(\bm{x}) \text{\ \ and \ }
\bm{x}_{\delta}^\star := \underset{\bm{x} \in \mathbb{R}^d} {\operatorname{argmin}} \hat{f}_\delta (\bm{x}).
\end{align*}
\end{dfn}

There are a total of $M$ smoothed functions in this paper. The largest noise level is $\delta_1$ and the smallest is $\delta_{M+1}=0$. Thus, $\hat{f}_{\delta_{M+1}}=f$.

\begin{dfn}\label{dfn:3.1}
Let $\delta_1 \in \mathbb{R}$. A function $f \colon \mathbb{R}^d \to \mathbb{R}$ is said to be ``new $\sigma$-nice'' if the following conditions hold$\colon$ 

{\em (i)} For all $m \in [M]$ and all $\gamma_m \in (0,1)$, there exist $\delta_m \in \mathbb{R}$ with $|\delta_{m+1}| := \gamma_m|\delta_m|$ and $\bm{x}_{\delta_m}^\star$ such that
\begin{align*}
\left\| \bm{x}_{\delta_m}^\star - \bm{x}_{\delta_{m+1}}^\star \right\| \leq |\delta_m| - |\delta_{m+1}|.
\end{align*}

{\em (ii)} For all $m \in [M]$ and all $\gamma_m \in (0,1)$, there exist $\delta_m \in \mathbb{R}$ with $|\delta_{m+1}| := \gamma_m|\delta_m|$ and $d_m > 1$ such that the function $\hat{f}_{\delta_m} (\bm{x})$ is $\sigma$-strongly convex on $N(\bm{x}^\star; d_m\delta_m)$.
\end{dfn}

\subsection{Assumptions and Lemmas}
\label{subsec:2.2}
We make the following assumptions: 
\begin{assum}\label{assum:03}
{\em (A1)} $f \colon \mathbb{R}^d \to \mathbb{R}$ is continuously differentiable and $L_g$-smooth, i.e., for all $\bm{x}, \bm{y} \in \mathbb{R}^d$,
%\begin{align*}
$\| \nabla f(\bm{x}) - \nabla f(\bm{y}) \| \leq L_g \| \bm{x} - \bm{y} \|.$
%\end{align*}

\noindent
{\em (A2)} $f \colon \mathbb{R}^d \to \mathbb{R}$ is an $L_f$-Lipschitz function, i.e., for all $\bm{x}, \bm{y} \in \mathbb{R}^d$,
%\begin{align*}
$\left|f(\bm{x}) - f(\bm{y}) \right| \leq L_f \| \bm{x} - \bm{y} \|.$
%\end{align*}

\noindent
{\em (A3)} Let $(\bm{x}_t)_{t \in \mathbb{N}} \subset \mathbb{R}^d$ be the sequence generated by an optimizer.

{\em (i)} For each iteration $t$, 
%\begin{align*}
$\mathbb{E}_{\xi_t}\left[ \mathsf{G}_{\xi_t}(\bm{x}_t) \right]
= \nabla f(\bm{x}_t).$
%\end{align*} 

{\em (ii)} There exists a nonnegative constant $C_{\text{opt}}^2$ for an optimizer such that
%\begin{align*}
$\mathbb{E}_{\xi_t}\left[ \| \mathsf{G}_{\xi_t}(\bm{x}_t) - \nabla f(\bm{x}_t) \|^2 \right]
\leq C_{\text{opt}}^2.$
%\end{align*}

\noindent
{\em (A4)} For each iteration $t$, the optimizer samples a mini-batch $\mathcal{S}_t \subset \mathcal{S}$ and estimates the full gradient $\nabla f$ as 
\begin{align*}
\nabla f_{\mathcal{S}_t} (\bm{x}_t)
:= \frac{1}{b} \sum_{i\in [b]} \mathsf{G}_{\xi_{t,i}}(\bm{x}_t)
= \frac{1}{b} \sum_{\{i\colon \bm{z}_i \in \mathcal{S}_t\}} \nabla f_i(\bm{x}_t).
\end{align*}
{\em (A5)} There exists a positive constant $K_{\text{\rm{opt}}}$ for an optimizer, for all $t \in \mathbb{N}$,
%\begin{align*}
$\mathbb{E} \left[ \| \nabla f (\bm{x}_t) \|^2 \right] \leq K_{\text{\rm{opt}}}^2.$
%\end{align*}

\begin{rem}
\rm{Three algorithms appear in this paper: SGD (Algorithm \ref{alg:sgd}), SHB (Algorithm \ref{alg:shb}), and NSHB (Algorithm \ref{alg:nshb}). In assumptions (A3)(ii) and (A5), $C_\text{opt}^2$ and $K_\text{opt}$ represent optimizer-specific constants. For example, $C_\text{SGD}^2$ is the variance of the stochastic gradient when $(\bm{x}_t)_{t \in \mathbb{N}}$ is the sequence generated by SGD.}
\end{rem}
\end{assum}

\section{Explicit Graduated Optimization}
Algorithm \ref{alg:gnc} \cite{Sato2023Usi} below is an embodiment of explicit graduated optimization, where the noise decay rate $\gamma_m := \frac{(M-m)^p}{\left\{ M-(m-1) \right\}^p}$ $(p \in(0,1])$ is that of a polynomial decay scheduler with power $p$, which is theoretically optimal in the sense that it satisfies the conditions necessary to ensure that the graduated optimization algorithm does not leave the strongly convex region of the new $\sigma$-nice function. Algorithm \ref{alg:sgd} is used to optimize each smoothed function. 
\begin{rem}
\rm{In explicit graduated optimization, each smoothed function can be optimized with any algorithm as long as its convergence to a local strongly convex function is guaranteed. We follow the prior work \cite{Sato2023Usi} and use SGD, which includes GD, in Algorithm \ref{alg:sgd}.}
\end{rem}

\begin{table*}[htbp]
\centering
\begin{tabular}{c|c|cccc}
\hline
function's reference &function & EGO(geo) & EGO(nice) & GA & PSO \\
\hline
\cite{Ackley1987Aco, Thomas1993AnO} &Ackley's & 1.80E+01 & \textbf{4.04E-03} & 1.52E+00 & 5.37E+00 \\
\cite{Rahnamayan2007Ano} &Alpine1 & 2.34E-01 & 1.25E-01 & 2.59E-01 & \textbf{2.72E-04} \\
\cite{Molga2005Tes} &Drop-Wave & 9.93E-01 & 9.94E-01 & 9.30E-01 & \textbf{9.10E-01} \\
\cite{Yang2010Eng} &Ellipsoid & 4.22E-01 & 4.82E-04 & 5.35E+00 & \textbf{1.12E-04} \\
\cite{Griewank1981Gen} &Griewank & 2.55E+00 & \textbf{2.32E-03} & 9.88E-03 & 3.86E-02 \\
\cite{Beyer2012Hap} &HappyCat &2.39E+01 & 1.76E+00 & 1.12E+00 & \textbf{5.81E-01} \\
\cite{Ying2016GPU} &HGBat &5.07E-01 & \textbf{5.04E-01} & 1.25E+00 & 5.86E-01 \\
\cite{Plevris2022ACo} &Modified Ridge & 6.63E+00 & 6.63E+00 & 1.73E+00 & \textbf{4.57E-01} \\
\cite{Torn1989Glo, Rudolph1990Glo} &Rastrigin's &1.85E+00 & \textbf{2.26E-02} & 2.44E+01 & 1.51E+02 \\
\cite{Rosenbrock1960AnA, Dixon1994Eff} &Rosenbrock's & 6.00E+98 & 2.44E+127 & 9.74E+01 & \textbf{9.53E+01} \\
\cite{Molga2005Tes} &Rotated Hyper-ellipsoid &5.20E-01 & 4.95E-04 & 6.73E+00 & \textbf{2.53E-05} \\
\cite{Salomon1996Re-} &Salomon &7.24E-01 & \textbf{2.02E-01} & 7.54E-01 & 2.11E+00 \\
\cite{Schaffer1985Mul, Caruana1989Ast} &Schaffer's F7 & 2.68E+01 & 1.12E+01 & \textbf{6.83E-01} & 1.80E+01 \\
\cite{Schwefel1981Num} &Schwefel &8.33E+03 & 8.33E+03 & \textbf{7.10E+03} & 7.20E+03 \\
\cite{Schwefel1981Num} &Schwefel 2.21 &3.76E+01 & \textbf{2.06E-02} & 2.17E+00 & 3.34E+01 \\
\cite{Schumer1968Ada} &Sphere &\textbf{6.98E-07} & 1.58E-05 & 2.34E-01 & 3.46E-06 \\
\hline
\end{tabular}
\vspace*{-5pt}
\caption{Average values of the optimum results over 50 runs, for Algorithm \ref{alg:gnc}, GA, and PSO, for dimensions $D=50$.}
\label{tab:1}
\vspace*{-15pt}
\end{table*}

\begin{algorithm}[H]
\caption{Explicit Graduated Optimization}
\label{alg:gnc}
\begin{algorithmic}
\REQUIRE
$\epsilon, r, \bar{d}, B_2, H_3, H_4 > 0, p \in (0,1], \newline \quad \quad \ \ \ \ \ \bm{x}_1 \in \mathbb{R}^d, b \in [n], \eta > 0$
\STATE{$\delta_1 := \frac{2L_f}{\sigma r}$}
\STATE{$\alpha_0 := \min \left\{ \frac{\sigma r}{8L_f^2 \left(1+\bar{d}\right)}, \frac{\sqrt{\sigma}r}{2\sqrt{2}L_f} \right\}, M^p := \frac{1}{\alpha_0 \epsilon}$}
\FOR{$m=1$ to $M+1$}
\IF{$m \neq M+1$}
\STATE{$\epsilon_m := \sigma \delta_m^2,\ T_F := H_4 / (\epsilon_m - H_3 \eta_m )$}
\STATE{$\gamma_m := \frac{(M-m)^p}{\left\{ M-(m-1) \right\}^p}$}
\ENDIF
\STATE{$\bm{x}_{m+1} := \text{SGD}(T_F, \bm{x}_m, \hat{f}_{\delta_m}, b, \eta)$}
\STATE{$\delta_{m+1} := \gamma_m \delta_m$}
\ENDFOR
\STATE \textbf{return} $\bm{x}_{M+2}$
\end{algorithmic}
\end{algorithm}
\vspace*{-10pt}
\begin{algorithm}[H]
\caption{SGD with constant learning rate}
\label{alg:sgd}
\begin{algorithmic}
\REQUIRE$T_F, \hat{\bm{x}}_1^{(m)} \in \mathbb{R}^d, \hat{f}_{\delta_m}, b \in [n], \eta > 0$
\FOR{$t=1$ to $T_F$}
\STATE{$\hat{\bm{x}}_{t+1}^{(m)} := \hat{\bm{x}}_t^{(m)} - \eta \nabla \hat{f}_{\delta_m, \mathcal{S}_t} (\hat{\bm{x}}_t)$}
\ENDFOR
\STATE \textbf{return} $\hat{\bm{x}}_{T_F +1} = \text{SGD}(T_F, \hat{\bm{x}}_1^{(m)}, \hat{f}_{\delta_m}, b, \eta)$
\end{algorithmic}
\end{algorithm}

If the objective function $f$ is a new $\sigma$-nice function, it is guaranteed that Algorithm \ref{alg:gnc} will converge to an $\epsilon$-neighborhood of the global optimal solution of $f$ in $\mathcal{O}\left( 1/\epsilon^{\frac{1}{p}} \right)$ rounds. Whether Algorithm \ref{alg:gnc} works effectively with test functions for global optimization has not been considered in a previous study. This paper deals with 16 traditional benchmark functions to verify the performance of global optimization algorithms (see Table \ref{tab:11} for the function's name and Appendix B for their definitions). For example, for all $\bm{x} := (x_1, \cdots, x_D)^\top \in \mathbb{R}^D$, Rastrigin's function is defined as follows:
\begin{align*}
f(\bm{x}) :=  \sum_{i=1}^{D} \left\{ x_i^2 -10\cos(2\pi x_i) \right\} + 10D.
\end{align*}

\vspace*{-8pt}
\begin{figure}[htbp]
\begin{minipage}[t]{0.99\hsize}
\centering
\includegraphics[width=0.8\textwidth]{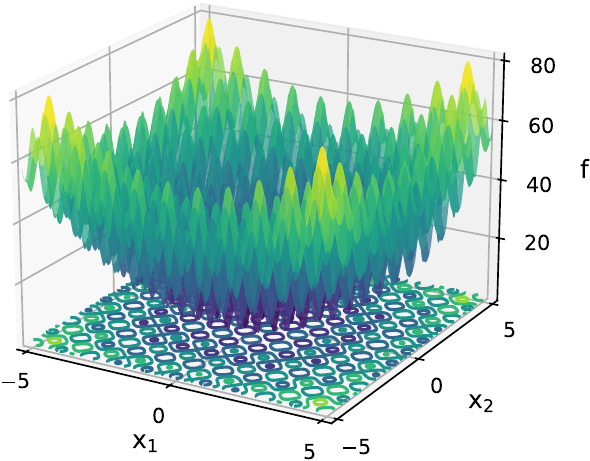}%
\vspace*{-6pt}
\caption{Rastrigin's function of two variables}
\label{fig:rast}
\end{minipage}
\end{figure}
\vspace*{-8pt}

Figure \ref{fig:rast} plots Rastrigin's function for $D=2$. Theorem \ref{thm:01} shows that this function, which has many local optimal solutions, is a new $\sigma$-nice function. The proof of the theorem is in Appendix A.

\begin{thm}\label{thm:01}
Rastrigin's function is a new $\sigma$-nice function.
\end{thm}

Theorem \ref{thm:01} implies that Algorithm \ref{alg:gnc} can achieve global optimum of Rastrigin's function. To test the algorithm's performance, we optimized the 16 benchmark functions with it (EGO) with optimal decay rates (nice) and with suboptimal decay rates (geo), where $\gamma_m := c$ $(c \in (0,1))$, i.e., the noise sequence is a geometric progression, and did the same with  two well-known global optimization algorithms, i.e., the genetic algorithm (GA) \cite{Holland1975Ada} and particle swarm optimization (PSO) \cite{Kennedy1995Par}. The dimension $D$ of the objective functions was set to $50$. 
The experimental environment was Intel Core i9 139000KF CPU.

Table \ref{tab:1} shows the average of the optimum results of $50$ runs starting from a random initial point. 
EGO with both decay rates used GD to optimize the smoothed function. The results for GA and PSO are taken from \cite{Plevris2022ACo}. Note that the optimal value for all functions is 0. See Appendix B for the search range of each objective function and the parameters of the Algorithms. The results in the table indicate that Algorithm \ref{alg:gnc} (EGO) significantly outperforms the other methods for several objective functions and that EGO with theoretically optimal decay rates (nice) generally performs better than EGO with suboptimal decay rates (geo). Thus, Algorithm \ref{alg:gnc} with optimal decay rates is effective for traditional global optimization benchmark functions.

However, explicit graduated optimization is not effective for DNNs, such as ResNet \cite{He2016Dee}. Figure \ref{fig:01} shows the results of using Algorithm \ref{alg:gnc} to train ResNet18 on the CIFAR100 dataset for 200 epochs. Note that it used SGD with $\nabla f_{\mathcal{S}_t} (\bm{x}_t + \delta_m \bm{u}_t)$ as the search direction for decreasing $(\delta_m)_{m \in [M]}$, where $\bm{u}_t \in \mathbb{R}^d$ is Gaussian noise and $M=200$; i.e., the noise $\delta_m$ was decreased each epoch.
The experimental environment for training DNN in this paper was as follows: NVIDIA GeForce RTX 4090 $\times$ 1GPU.

\begin{figure}[htbp]
\begin{minipage}[t]{0.99\hsize}
\centering
\includegraphics[width=1\textwidth]{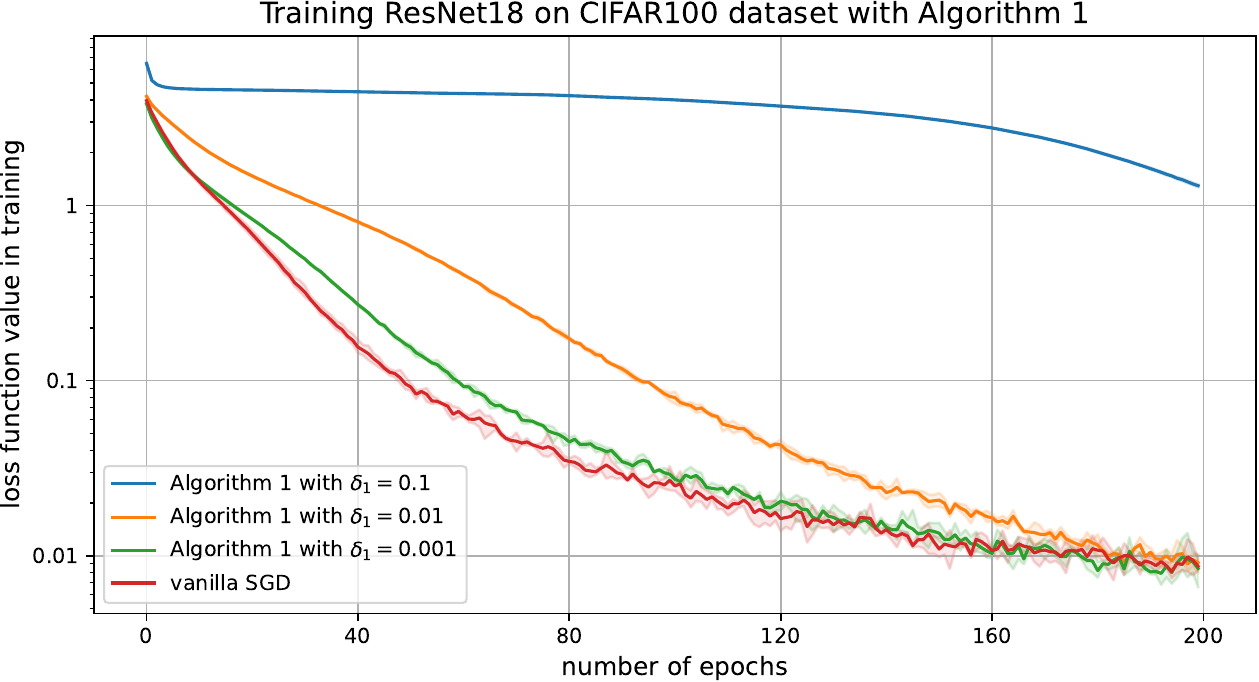}%
\vspace*{-8pt}
\caption{Loss function value for training versus the number of epochs in training ResNet18 on the CIFAR100 dataset. The solid lines represent the mean value, and the shaded areas represent the maximum and minimum values over three runs.}
\label{fig:01}
\end{minipage}
\vspace*{-15pt}
\end{figure}

Compared with vanilla SGD without noise, Algorithm \ref{alg:gnc} with initial noise values $\delta_1$ of 0.1, 0.01, and 0.001 does not achieve significantly lower loss function values; instead, the noise simply interferes with training. In fact, the loss function values are comparable to those of vanilla SGD only when the noise is nearly zero, and it does not fulfill its role of avoiding locally optimal solutions and yielding lower loss function values. 
The cause is theoretically unknown, but the experimental validity of similar methods in the previous study \cite{Elad2016OnG} suggests that the number of parameters in the model is a factor, i.e., that explicit graduated optimization may not be effective for very high dimensional functions. The number of model parameters in the previous study is 23,860, while ours is about 11.2M.
Because of these limitations of explicit graduated optimization, we do not see any benefit to applying it to DNNs.

\section{Implicit Graduated Optimization}
\textbf{Implicit graduated optimization with SGD.} As mentioned above, \cite{Sato2023Usi} showed that simply using SGD generates stochastic noise of level $\delta^{\text{SGD}} = \eta C_{\text{SGD}}^2 /\sqrt{b}$, which smooths the objective function $f$.

Here, let $\bm{y}_{t}$ be the parameter updated by gradient descent (GD) at time $t$, and let $\bm{x}_{t+1}$ be the parameter updated by SGD, i.e., 
\begin{align*}
\bm{y}_{t} &:= \bm{x}_t - \eta \nabla f(\bm{x}_t), \\
\bm{x}_{t+1} &:= \bm{x}_t - \eta \nabla f_{\mathcal{S}_t}(\bm{x}_t) \\
 &= \bm{x}_t - \eta (\nabla f(\bm{x}_t) + \bm{\omega}_t^{\text{SGD}}),
\end{align*}
where $\bm{\omega}_t^{\text{SGD}} := \nabla f_{\mathcal{S}_t} (\bm{x}_t) - \nabla f(\bm{x}_t)$ is stochastic noise of SGD. Then, using Definition \ref{dfn:fhat} and $\mathbb{E}_{\xi_t} \left[ \left\| \bm{\omega}_t^{\text{SGD}} \right\|^2 \right] \leq \frac{C_{\text{SGD}}}{\sqrt{b}}$ which holds under Assumptions (A3)(ii) and (A4), we have
\begin{align}
\mathbb{E}_{\bm{\omega}_t} \left[\bm{y}_{t+1} \right] 
= \mathbb{E}_{\bm{\omega}_t} \left[ \bm{y}_t \right]- \eta \nabla \hat{f}_{\frac{\eta C}{\sqrt{b}}}(\bm{y}_t). \label{eq:2}
\end{align}
See Section 3.3 and Appendix A of \cite{Sato2023Usi} for details on the derivation of equation (\ref{eq:2}).

Therefore, optimizing the objective function $f$ by SGD is equivalent to optimizing its smoothed version $\hat{f}_{\delta^{\text{SGD}}}$ by GD in the sense of expectation. On the basis of this fact, \cite{Sato2023Usi} proposed the following algorithm that implicitly achieves graduated optimization by varying the learning rate $\eta$ and batch size $b$ so as to decrease the degree of smoothing $\delta^{\text{SGD}}$ during training. That is, they decreased the learning rate $\eta$ and increased the batch size $b$ during training, thereby decreasing the degree of smoothing $\delta^{\text{SGD}}$. Algorithm \ref{alg:gnc2} is an implicit graduated optimization algorithm that exploits the natural stochastic noise of SGD. The decay rate of $\delta^{\text{SGD}}$ is set so that the combined effect of decreasing the learning rate $\kappa_m \in (0,1]$ and increasing the batch size $\lambda_m \geq 1$ is $\gamma_m := \frac{(M-m)^p}{\{ M-(m-1)\}^p}$. Algorithm \ref{alg:sgd} (SGD) is used to optimize each smoothed function. 

\vspace*{-8pt}
\begin{algorithm}[H]
\caption{Implicit Graduated Optimization with SGD}
\label{alg:gnc2}
\begin{algorithmic}
\REQUIRE
$\epsilon, r, \bar{d}, B_2, H_3, H_4 > 0, p \in (0,1], \newline \quad \quad \ \ \ \ \ \bm{x}_1 \in \mathbb{R}^d, b_1 \in [n], \eta_1 > 0$
\STATE{$\delta_1 := \frac{\eta_1 C}{\sqrt{b_1}}$}
\STATE{$\alpha_0 := \min \left\{ \frac{\sqrt{b_1}}{4L_f \eta_1C \left(1+\bar{d}\right)}, \frac{\sqrt{b_1}}{\sqrt{2\sigma} \eta_1 C} \right\}, M^p:= \frac{1}{\alpha_0 \epsilon}$}
\FOR{$m=1$ to $M+1$}
\IF{$m \neq M+1$}
\STATE{$\epsilon_m := \sigma^2 \delta_m^2, \ T_F := H_4 / (\epsilon_m - H_3 \eta_m )$}
\STATE{$\gamma_m := \frac{(M-m)^p}{\left\{ M-(m-1)\right\}^p}$}
\STATE{$\kappa_m / \sqrt{\lambda_m} = \gamma_m \ (\kappa_m \in (0,1], \lambda_m \geq 1)$}
\ENDIF
\STATE{$\bm{x}_{m+1} := \text{SGD}(T_F, \bm{x}_m, \hat{f}_{\delta_m}, \eta_m, b_m)$}
\STATE{$\eta_{m+1} := \kappa_m \eta_m, b_{m+1} := \lambda_m b_m$}
\STATE{$\delta_{m+1} := \frac{\eta_{m+1} C}{\sqrt{b_{m+1}}}$}
\ENDFOR
\STATE \textbf{return} $\bm{x}_{M +2}$
\end{algorithmic}
\end{algorithm}
\vspace*{-8pt}

\begin{rem}
\rm{In implicit graduated optimization, the algorithm that optimizes each smoothed function must be GD or SGD. Unlike explicit graduated optimization, which explicitly provides a smoothed function, based on (\ref{eq:2}), the smoothed function can only be optimized by a GD- or SGD-like algorithm. It is unclear which, GD or SGD, is better. Our analysis is valid in both cases, since we provide an analysis of SGD, an extension of GD.}
\end{rem}

If the objective function $f$ is a new $\sigma$-nice function, it is guaranteed that Algorithm \ref{alg:gnc2} converges to an $\epsilon$-neighborhood of the global optimal solution of $f$ in $\mathcal{O}\left( 1/\epsilon^{\frac{1}{p}} \right)$ rounds. Since the prior study lacked numerical experiments to verify the performance of Algorithm \ref{alg:gnc2}, here, we provide experimental results that are faithful to the decay rate $\gamma_m$ employed in Algorithm \ref{alg:gnc2}. Figure \ref{fig:03} shows the results of training ResNet18 on the CIFAR100 dataset with Algorithm \ref{alg:gnc2} with a constant learning rate and constant batch size, i.e. $\kappa_m=1, \lambda_m=1$ (method 1, i.e., vanilla SGD), a decaying learning rate and constant batch size, i.e., $\kappa_m=\gamma_m$ and $\lambda_m=1$ (method 2), a constant learning rate and increasing batch size, $\kappa_m=1, \lambda_m=1/\gamma_m^{13/10}$ (method 3), and a hybrid setting, i.e., $\kappa_m=\gamma_m^{4/9}$ and $\lambda_m=1/\gamma_m$, where the decay rate is set to $\gamma_m = \frac{(M-m)^{0.9}}{\left\{ M - (m-1)\right\}^{0.9}}\ (M=200, m \in [M])$ (method 4). In a 200-epoch training, methods 2, 3, and 4 update the hyperparameters every epoch. In method 1, the learning rate and the batch size are fixed at 0.1 and 128, respectively. In method 2, the initial learning rate is 0.1 and the batch size is fixed at 128. In method 3, the learning rate is fixed at 0.1 and the initial batch size is 32. In method 4, the initial learning rate is 0.1 and the initial batch size is 32. 

\begin{figure}[htbp]
\begin{minipage}[t]{0.99\hsize}
\centering
\includegraphics[width=1\textwidth]{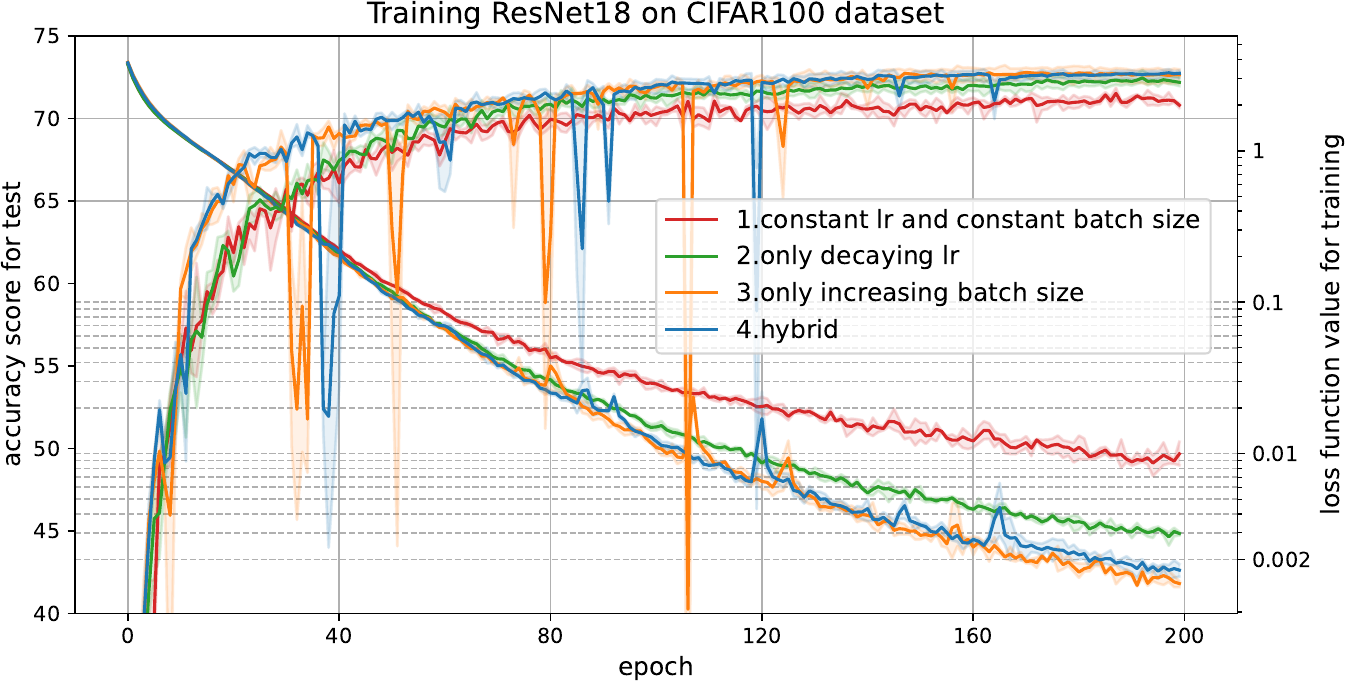}%
\caption{Accuracy score in testing and loss function value in training versus the number of epochs in training ResNet18 on the CIFAR100 dataset with SGD. The solid lines represent the mean value, and the shaded areas represent the maximum and minimum values over three runs.}
\label{fig:03}
\end{minipage}
\vspace*{-5pt}
\end{figure}

Figure \ref{fig:03} indicates that methods 2, 3, and 4 yield a lower loss function and higher test accuracy than those of vanilla SGD (method 1). In particular, the fact that methods 3 and 4, which increase the batch size, are superior in both loss function value and test accuracy suggests that a high learning rate is necessary for successful training. In fact, the loss curve for method 2, which decreases the learning rate, is nearly identical to that of methods 3 and 4 until the middle of training.

\textbf{Implicit graduated optimization with SGD with momentum.} There are a number of variants of SGD with momentum; this paper focuses on the simplest ones, i.e., the stochastic heavy ball (SHB) method and normalized stochastic heavy ball (NSHB) method, which are defined as follows:

\vspace*{-5pt}
\begin{algorithm}[H]%
\caption{Stochastic Heavy Ball (SHB)}
\label{alg:shb}
\begin{algorithmic}
\REQUIRE $\bm{x}_0, \eta>0, \beta \in [0,1), \bm{m}_{-1} := \bm{0}$
\FOR{$t=0$ to $T-1$}
\STATE $\bm{m}_{t} := \nabla f_{\mathcal{S}_t} (\bm{x}_t) + \beta \bm{m}_{t-1}$
\STATE $\bm{x}_{t+1} := \bm{x}_t - \eta \bm{m}_t$
\ENDFOR
\RETURN $\bm{x}_T$
\end{algorithmic}
\end{algorithm}
\vspace*{-15pt}
\begin{algorithm}[H]
\caption{Normalized Stochastic Heavy Ball (NSHB)}
\label{alg:nshb}
\begin{algorithmic}
\REQUIRE $\bm{x}_0, \eta>0, \beta \in [0,1), \bm{d}_{-1} := \bm{0}$
\FOR{$t=0$ to $T-1$}
\STATE $\bm{d}_{t} := (1-\beta)\nabla f_{\mathcal{S}_t} (\bm{x}_t) + \beta \bm{d}_{t-1}$
\STATE $\bm{x}_{t+1} := \bm{x}_t - \eta \bm{d}_t$
\ENDFOR
\RETURN $\bm{x}_T$
\end{algorithmic}
\end{algorithm}
\vspace*{-10pt}

\cite{Sato2024Rol} showed that the SHB and NSHB methods each generate stochastic noise of level,
\begin{align*}
&\delta^{\text{\rm{SHB}}} = \eta \sqrt{\left(1 + \hat{\beta} \right)\frac{C_{\text{SHB}}^2}{b} + \hat{\beta} K_{\text{SHB}}^2}, \\
&\delta^{\text{\rm{NSHB}}} = \eta \sqrt{\frac{1}{1-\beta} \frac{C_{\text{\rm{NSHB}}}^2}{b}},
\end{align*}
which smooths the objective function $f$, where $\hat{\beta} := \frac{\beta(\beta^2 - \beta + 1)}{(1-\beta)^2}$. That is, at time $t$, let $\bm{y}_{t}$ be the parameter updated by GD and $\bm{z}_{t+1}$ be the parameter updated by SHB. Then, the following holds:
\begin{align}
\mathbb{E}_{\bm{\omega}_t^{\text{SHB}}} \left[\bm{y}_{t+1} \right] 
= \mathbb{E}_{\bm{\omega}_t^{\text{SHB}}} \left[ \bm{y}_t \right]- \eta \nabla \hat{f}_{\delta^{\text{SHB}}}(\bm{y}_t),\label{eq:3}
\end{align}
where $\bm{\omega}_t^{\text{SHB}} := \bm{m}_t - \nabla f(\bm{x}_t)$ is stochastic noise of SHB. Since optimizing the objective function $f$ with SHB or NSHB is equivalent (in the sense of expectation) to optimizing its smoothed version, $\hat{f}_{\delta^{\text{SHB}}}$ and $\hat{f}_{\delta^{\text{NSHB}}}$, with GD, we can use the natural stochastic noise of SHB and NSHB to construct an implicit graduated optimization algorithm. %That is, the learning rate $\eta$ and momentum factor $\beta$ are decreased and the batch size $b$ is increased during training, thereby decreasing the degree for smoothing $\delta^{\text{SHB}}$ and $\delta^{\text{NSHB}}$. 
That is, the degree of smoothing $\delta^{\text{SHB}}$ and $\delta^{\text{NSHB}}$ are decreased by decreasing the learning rate $\eta$ and momentum factor $\beta$ and increasing the batch size $b$ during training.

\begin{algorithm}
\caption{Implicit Graduated Optimization with SHB}
\label{alg:gnc3}
\begin{algorithmic}
\REQUIRE$\epsilon >0, p \in (0,1], \bar{d} > 0, \bm{x}_1 \in \mathbb{R}^d, \eta_1 > 0, \newline \quad \quad \ \ \ \ \  b_1 \in [n], \beta_1 \in [0,1), C_{\text{\rm{SHB}}}^2 > 0, K_{\text{\rm{SHB}}} > 0$
\STATE{$\hat{\beta}_t := \frac{\beta_t(\beta_t^2 - \beta_t + 1)}{(1-\beta_t)^2} \ (\forall t \in \mathbb{N})$}
\STATE{$\delta_1^{\text{\rm{SHB}}} := \eta_1 \sqrt{(1 + \hat{\beta}_1)\frac{C_{\text{\rm{SHB}}}^2}{b_1} + \hat{\beta}_1 K_{\text{\rm{SHB}}}^2}$}
\STATE{$\alpha_0 := \min \left\{ \frac{1}{4L_f \left| \delta_1^{\text{\rm{SHB}}} \right| \left(1+\bar{d}\right)}, \frac{1}{\sqrt{2}\sigma \left| \delta_1^{\text{\rm{SHB}}}\right|} \right\}, M^p:= \frac{1}{\alpha_0 \epsilon}$}
\FOR{$m=1$ to $M+1$}
\IF{$m \neq M+1$}
\STATE{$\epsilon_m := \sigma^2 {\delta_m^{\text{\rm{SHB}}}}^2, \ T_F := H_4 / \left( \epsilon_m - H_3 \eta_m \right)$}
\STATE{$\gamma_m := \frac{(M-m)^p}{\left\{ M-(m-1)\right\}^p}$}
\STATE{$\frac{\kappa_m \eta_m \sqrt{(1 + \rho_m \hat{\beta}_m) \frac{C_{\text{\rm{SHB}}}^2}{\lambda_m b_m} + \rho_m \hat{\beta}_m K_{\text{\rm{SHB}}}^2}}{\eta_m \sqrt{(1 + \hat{\beta}_m) \frac{C_{\text{\rm{SHB}}}^2}{b_m} + \hat{\beta}_m K_{\text{\rm{SHB}}}^2}} = \gamma_m$ \\\quad\qquad\qquad\qquad\qquad$(\kappa_m, \rho_m \in (0,1], \lambda_m \geq 1)$}
\ENDIF
\STATE{$\bm{x}_{m+1} := \text{SGD}(T_F, \bm{x}_m, \hat{f}_{\delta_m}, \eta_m, b_m)$}
\STATE{$\eta_{m+1} := \kappa_m \eta_m, b_{m+1} := \lambda_m b_m, \hat{\beta}_{m+1} := \rho_m \hat{\beta}_m$}
\STATE{$\delta_{m+1}^{\text{\rm{SHB}}} := \eta_{m+1} \sqrt{(1 + \hat{\beta}_{m+1})\frac{C_{\text{\rm{SHB}}}^2}{b_{m+1}} + \hat{\beta}_{m+1} K_{\text{\rm{SHB}}}^2}$}
\ENDFOR
\RETURN $\bm{x}_{M +2}$
\end{algorithmic}
\end{algorithm}

Algorithm \ref{alg:gnc3} is an implicit graduated optimization that exploits the natural stochastic noise of SHB. The decay rate of $\delta^{\text{SHB}}$ is set so that the combined effect of decreasing the learning rate $\kappa_m \in (0,1]$ and momentum factor $\rho \in (0,1]$ and increasing the batch size $\lambda_m \geq 1$ is $\gamma_m := \frac{(M-m)^p}{\{ M-(m-1)\}^p}$.

The following theorem guarantees the convergence of Algorithm \ref{alg:gnc3} for the new $\sigma$-nice function (See Appendix C for details on the proof of Theorem \ref{thm:02}).

\begin{thm}
[Convergence analysis of Algorithm \ref{alg:gnc3}]\label{thm:02} Let $\epsilon \in (0, 1]$ and $f$ be an $L_f$-Lipschitz new $\sigma$-nice function. Suppose that we apply Algorithm \ref{alg:gnc2}; after $\mathcal{O}\left( 1/\epsilon^{\frac{1}{p}}\right)$ rounds. Then, the algorithm reaches an $\epsilon$-neighborhood of the global optimal solution $\bm{x}^\star$.
\end{thm}

To test the ability of Algorithm \ref{alg:gnc3} to reduce stochastic noise, we compared it with a vanilla SHB method in which the learning rate, batch size, and momentum are all constant. Here, Algorithm \ref{alg:gnc3} used a noise reduction method in which the hyperparameters are updated to reduce the degree of smoothing $\delta^{\text{\rm{SHB}}}$. 

\begin{figure}[htbp]
\begin{minipage}[t]{0.99\hsize}
\centering
\includegraphics[width=1\textwidth]{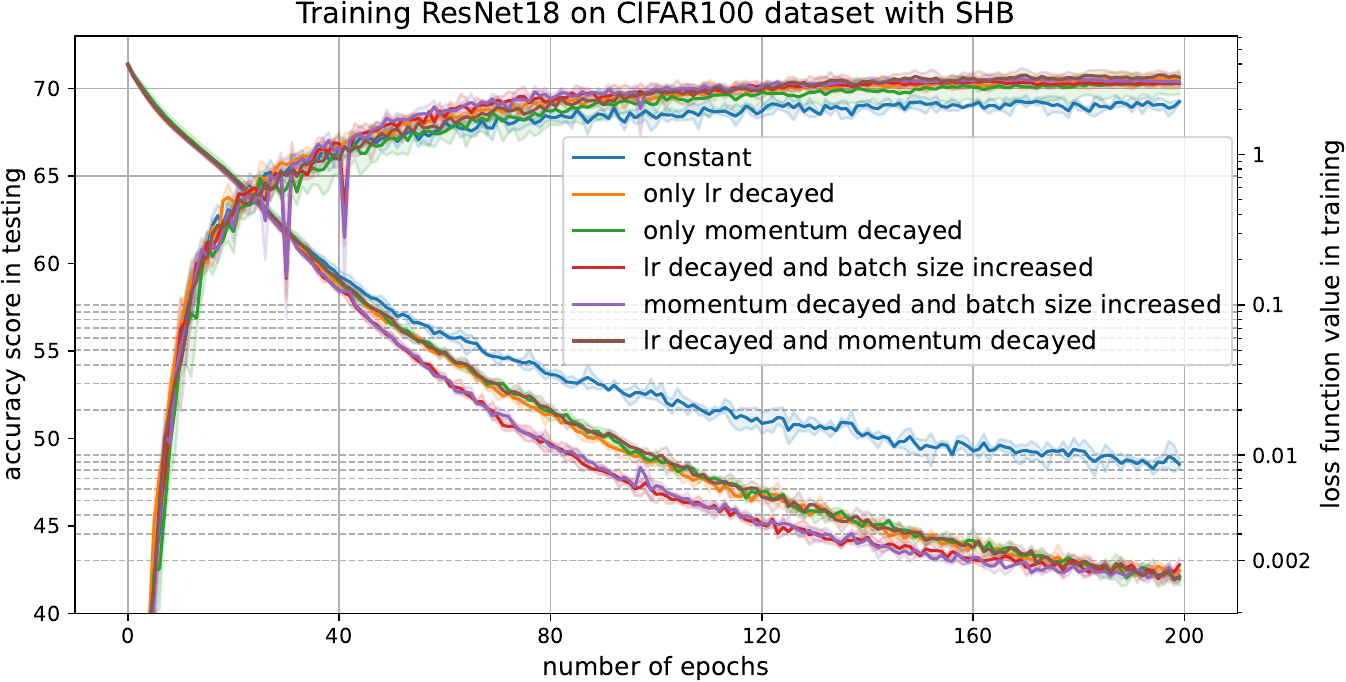}%
\caption{Accuracy score in testing and loss function value in training ResNet18 on the CIFAR100 dataset with Algorithm \ref{alg:gnc3} versus the number of epochs. The blue plot represents vanilla SHB, and the other five plots represent Algorithm \ref{alg:gnc3}. ``lr'' means the learning rate. The solid lines represent the mean value, and the shaded areas represent the maximum and minimum values over three runs.}
\label{fig:04}
\end{minipage}
\end{figure}

Figure \ref{fig:04} plots the accuracy in testing and the loss function value in training ResNet18 on the CIFAR100 dataset with SHB versus the number of epochs. Here, Algorithm \ref{alg:gnc3} outperformed vanilla SHB in both test accuracy and loss function value, thereby demonstrating that it is superior to SGD with momentum using constant parameters on image classification tasks. We confirmed similar results in training WideResNet-28-10 on the CIFAR100 dataset with SHB (see Figure \ref{fig:08} in Appendix D). See Figure \ref{fig:999} for the results of comparative experiments with warmup methods and others.

Appendix E analyzes the implicit graduated optimization algorithm using stochastic noise in NSHB and presents the results of numerical experiments.

\noindent
\textbf{Optimal decaying learning rates for SGD with momentum.} In graduated optimization, the optimal decay rate of noise is $\gamma_m := \frac{(M-1)^p}{\{ M-(m-1)\}^p}$ $(p \in (0,1])$ \cite{Sato2023Usi}. In implicit graduated optimization, the stochastic noise $\delta^{\text{opt}}$ is proportional to the learning rate $\eta$, so the optimal learning rate decay rate is likewise $\gamma_m$, i.e., a polynomial decay with a power less than $p = 1$. \cite{Sato2023Usi} conducted experiments on the optimal learning rate for SGD and found that a polynomial decay with a power less than $p=1$ achieves the smallest loss function value and yields the highest test accuracy, as theory suggests.

\begin{figure}[htbp]
\begin{minipage}[t]{0.99\hsize}
\centering
\includegraphics[width=1\textwidth]{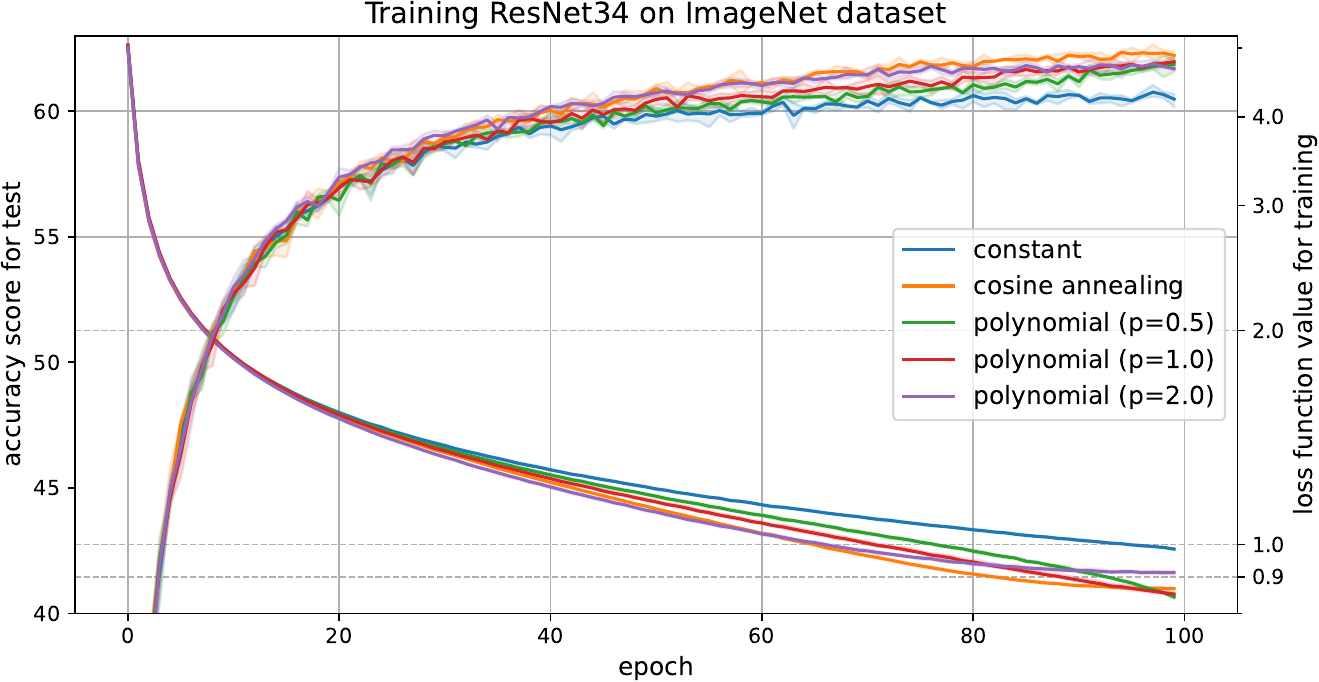}%
\caption{Accuracy score in testing and loss function value in training of ResNet34 on ImageNet dataset with SHB (Algorithm \ref{alg:shb}) versus the number of epochs. The solid lines represent the mean value, and the shaded areas represent the maximum and minimum values over three runs.}
\label{fig:05}
\end{minipage}
\end{figure}

In order to see which learning rate scheduler gives the smallest loss function value for SGD with momentum, we trained ResNet34 \cite{He2016Dee} on the ImageNet dataset \cite{Deng2009Ima} with SHB for 100 epochs. The results in Figure \ref{fig:05} indicate that a polynomial decay with a power less than or equal to 1 achieves the smallest training loss function.
We confirmed similar results in training ResNet18 and WideResNet-28-10 on the CIFAR100 dataset with SHB and NSHB (see Figure \ref{fig:13}-\ref{fig:16} in Appendix F).

\section{Conclusion}
We showed that Rastrigin's function, one of the classical benchmark functions, is a new $\sigma$-nice function and that explicit graduated optimization with optimal noise scheduling is valid for traditional benchmark functions, but not effective for DNNs. Then, we developed an implicit graduated optimization algorithm using stochastic noise in SGD with momentum and analyzed its convergence. Finally, we conducted experiments on ImageNet on the optimal decaying learning rate scheduler for SGD with momentum and found that the theoretically optimal polynomial decay scheduler with a small power achieves the lowest loss function value. The most important advantage of our approach is that we can theoretically guarantee that SGD with momentum will converge to a global optimal solution only by increasing or decreasing hyperparameters. This has not been achieved by any existing methods. The contributions that were made by extending the implicit graduated optimization approach of SGD to that of SGD with momentum with stochastic noise include the following:
\begin{itemize}
\item Previous theories have been able to guarantee convergence to a local optimal solution for both $\beta \to 0$ and $\beta \to 1$\cite{Gitman2019Und}. By introducing the momentum factor into the implicit graduated optimization approach, we theoretically show that $\beta \to 0$ contributes to convergence to the global optimal solution, and we demonstrate its effectiveness in experiments.
\item From the fact that the degree of smoothing in SGD with momentum is determined by hyperparameters (see Figure 1), when graduated optimization is introduced, the knowledge of the optimal decay rate of the degree of smoothing leads directly to the knowledge of the optimal decay of hyperparameters such as the learning rate and momentum factor.
\end{itemize}
These findings can only be derived from an implicit graduated optimization perspective, and such a theory is expected to help users set hyperparameters intuitively.

%\section{Future Work}
%We have examined the effectiveness of stepwise optimization in training DNNs and its limitations.
%We have shown that implicit graduated optimization, in contrast to explicit graduated optimization, works well for training DNNs, yielding significantly lower loss function values and slightly higher test accuracy.
%We have dealt with the general empirical risk minimization (ERM) problem, which suggests that, in general, sufficiently low loss function values do not necessarily lead to sufficiently high generalization performance, but rather that the smoothness of the function around the approximate solution to which the optimizer converges is closely related to the generalization performance \cite{Shirish2017OnL, Izmaliov2018Ave, Li2018Vis}.
%Therefore, to take full advantage of the power of global optimization by the graduated optimization algorithm, it may be useful to apply the algorithm to problems where low loss function values lead to high generalization performance, such as Sharpness-Aware Minimization (SAM) \cite{Foret2021Sha}.
%Applying the graduated optimization algorithm to such non-ERM loss minimization problems is an interesting future work and  may help further development of machine learning and artificial intelligence.

\section{Acknowledgements}
We are sincerely grateful to Program Chairs, Area Chairs, and anonymous reviewers for helping us improve the original manuscript. This research is partly supported by the computational resources of the DGX A100 named TAIHO at Meiji University.
This work was supported by the Japan Society for the Promotion of Science (JSPS) KAKENHI Grant Number 24K14846 awarded to Hideaki Iiduka.

\bibliography{aaai25}

\clearpage
\onecolumn
\section{A. Proof of Theorem \ref{thm:01}}
\label{sec:A}
\begin{proof}
Recall that Rastrigin's function is defined as 
\begin{align*}
f(\bm{x}) 
&:=  \sum_{i=1}^{D} \left\{ x_i^2 -10\cos(2\pi x_i) \right\} + 10D \\
&= \| \bm{x} \|^2 -10\sum_{i=1}^{D}\cos(2\pi x_i) + 10D
\end{align*}
First, we need to derive a smoothed version of Rastrigin's function $\hat{f}_{\delta}(\bm{x})$.
\begin{align*}
\hat{f}_{\delta}(\bm{x}) 
&:= \mathbb{E}_{\bm{u} \sim \mathcal{N}\left(\bm{0}, \frac{1}{\sqrt{D}}I_D \right)} \left[ f(\bm{x}+\delta \bm{u})\right] \\
&= \mathbb{E}_{\bm{u} \sim \mathcal{N}\left(\bm{0}, \frac{1}{\sqrt{D}}I_D \right)} \left[ \| \bm{x}+\delta \bm{u} \|^2 - 10\sum_{i=1}^{D} \cos\left\{ 2\pi(x_i + \delta u_i)\right\} +10D\right] \\
&= \mathbb{E}_{\bm{u} \sim \mathcal{N}\left(\bm{0}, \frac{1}{\sqrt{D}}I_D \right)} \left[ \| \bm{x} \|^2 + 2\delta \langle \bm{x}, \bm{u}\rangle + \delta^2 \| \bm{u} \|^2 \right] - 10\sum_{i=1}^{D} \mathbb{E}_{u_i \sim \mathcal{N}\left(0, \frac{1}{\sqrt{D}} \right)} \left[ \cos\left\{ 2\pi(x_i + \delta u_i)\right\} \right] + 10D \\
&= \| \bm{x} \|^2 - 10\sum_{i=1}^{D} \mathbb{E}_{u_i \sim \mathcal{N}\left(0, \frac{1}{\sqrt{D}} \right)} \left[ \cos\left\{ 2\pi(x_i + \delta u_i)\right\} \right] + 10D + C,
\end{align*}
where we use $C := \delta^2 \mathbb{E}_{\bm{u}}\left[ \| \bm{u} \|^2 \right], \mathbb{E}_{\bm{u}}\left[ \bm{u} \right] = \bm{0}$. Now, let $y_i := x_i + \delta u_i$. Since $y_i \sim \mathcal{N}\left( x_i, \delta^2 \right)$, we have
\begin{align*}
\mathbb{E}_{u_i \sim \mathcal{N}\left(0, \frac{1}{\sqrt{D}} \right)} \left[ \cos\left\{ 2\pi(x_i + \delta u_i)\right\} \right]
&= \mathbb{E}_{y_i} \left[ \cos(2\pi y_i) \right] \\
&= \int_{-\infty}^{\infty} \cos\left\{ 2\pi y_i\right\} \frac{1}{\sqrt{2\pi\delta^2}} \exp\left( -\frac{(y_i-x_i)^2}{2\delta^2}\right) dy_i.
\end{align*}
Next, let $z_i := \frac{y_i -x_i}{\delta}$. From $dy_i = \delta dz_i$, we have
\begin{align*}
\mathbb{E}_{u_i \sim \mathcal{N}\left(0, \frac{1}{\sqrt{D}} \right)} \left[ \cos\left\{ 2\pi(x_i + \delta u_i)\right\} \right]
&= \frac{1}{\sqrt{2\pi}} \int_{-\infty}^{\infty} \cos\left\{ 2\pi(x_i + z_i \delta)\right\} \exp\left( -\frac{z_i^2}{2}\right) dz_i \\
&= \frac{1}{\sqrt{2\pi}} \int_{-\infty}^{\infty} \cos(2\pi x_i)\cos(2\pi z_i \delta) \exp\left( -\frac{z_i^2}{2}\right) dz_i \\
&\quad- \frac{1}{\sqrt{2\pi}} \int_{-\infty}^{\infty} \sin(2\pi x_i)\sin(2\pi z_i \delta) \exp\left( -\frac{z_i^2}{2}\right) dz_i  \\
&=\frac{\cos(2\pi x_i)}{\sqrt{2\pi}} \int_{-\infty}^{\infty} \cos(2\pi z_i \delta) \exp\left( -\frac{z_i^2}{2}\right) dz_i \\
&\quad- \frac{\sin(2\pi x_i)}{\sqrt{2\pi}} \int_{-\infty}^{\infty} \sin(2\pi z_i \delta) \exp\left( -\frac{z_i^2}{2}\right) dz_i \\
&=\frac{\cos(2\pi x_i)}{\sqrt{2\pi}} \int_{-\infty}^{\infty} \cos(2\pi z_i \delta) \exp\left( -\frac{z_i^2}{2}\right) dz_i \\
&=\frac{\cos(2\pi x_i)}{\sqrt{2\pi}} \sqrt{2\pi}\exp\left( -\frac{4\pi^2 \delta^2}{2}\right) \\
&=\cos(2\pi x_i) \exp\left(-2\pi^2 \delta^2 \right).
\end{align*}
Therefore,
\begin{align*}
\hat{f}_{\delta}(\bm{x}) = \| \bm{x} \|^2 - 10 \exp\left( -2\pi \delta^2 \right) \sum_{i=1}^{D} \cos(2\pi x_i) + 10D + C.
\end{align*}
To prove that Rastrigin's function $f$ is a new $\sigma$-nice function, it is sufficient to show that the following one-dimensional function $g$ is a new $\sigma$-nice function:
\begin{align*}
g(x) := x^2 - 10\cos(2\pi x)+10.
\end{align*}
Also, a smoothed version of $g$ can be derived as
\begin{align*}
\hat{g}_{\delta}(x) := x^2 - 10\exp\left(-2\pi \delta^2 \right) \cos(2\pi x)+10.
\end{align*}
For any $\delta \in \mathbb{R}$, $g$ has a global minimum at $x=0$, so $x_{\delta}^\star = 0$ holds for all $\delta \in \mathbb{R}$. Thus, for all $\delta_m \in \mathbb{R}$, 
\begin{align*}
\left|x_{\delta_m}^\star - x_{\delta_{m+1}}^\star \right| = |0-0| = 0 \leq |\delta_m| - |\delta_{m+1}|,
\end{align*}
which implies the first condition of new $\sigma$-niceness holds. Next, we show that the second condition of new $\sigma$-niceness holds, i.e., that a smoothed function $\hat{g}_{\delta}$ is $\sigma$-strongly convex in $N(x^\star; d_m \delta_m) (d_m > 1)$. The derivative of $\hat{g}_{\delta}$ is obtained as follows:
\begin{align*}
\hat{g}_{\delta}'(\bm{x}) &= 2x + 20\pi \exp(-2\pi \delta^2)\sin(2\pi x) \\
\hat{g}_{\delta}''(\bm{x}) &= 2 + 40\pi^2 \exp(-2\pi \delta^2)\cos(2\pi x) \\
\hat{g}_{\delta}'''(\bm{x}) &= -80\pi^3 \exp(-2\pi \delta^2)\sin(2\pi x).
\end{align*}
When $\delta > \delta^\star := \sqrt{\frac{1}{2\pi}\log(20\pi^2)} \approx 0.917$, $\hat{g}_{\delta}''(x)>0$ holds. Thus, for all $\delta > \delta^\star$ and all $x \in \mathbb{R}$, the smoothed function $\hat{g}_{\delta}(x)$ is $(2-40\pi^2\exp(-2\pi \delta^2))$-strongly convex. For example, $\hat{g}_{\delta}$ is $1.26$-strongly convex when $\delta=1.0$.

$\hat{g}_{\delta}(x)'''=0$ when $x=0, \pm0.5, \pm1, \pm1.5, \cdots$. Hence, when $\delta \leq \delta^\star$, $\hat{g}_{\delta}'''$ has multiple intersections with the $x$-axis. The $x$-coordinate of those intersections that is closest to $x=0$ is the maximum radius $r^+$ of the strongly convex region centered at $x^\star$ of $\hat{g}_{\delta}$. Solving for $\hat{g}_{\delta}''=0$, we find that 
\begin{align*}
r^+ = \frac{1}{2\pi} \arccos\left( -\frac{1}{20\pi^2 \exp(-2\pi \delta^2)} \right).
\end{align*}
Thus, $r^+ \approx 0.2507$ when $\delta=0$ and $r^+ \approx 0.5$ when $\delta=\delta^\star$. Therefore, if $\delta_1=0.25, d_m \approx1$, then $\hat{g}_{\delta}''(d_m\delta_m)$ is $\hat{g}_{\delta}''(d_m\delta_m)$-strongly convex in $N(x^\star; d_m\delta_m)$. Since $\hat{g}_{\delta}''(d_m\delta_m)$ is monotonically decreasing with respect to $\delta_m$, its minimum value is $\hat{g}_{\delta}''(d_m\delta_m)=\hat{g}_{0.25}''(0.25)=2$. Therefore, for all $i \in [D]$, $g(x_i)$ is $2$-strongly convex in $N(x^\star; d_m\delta_m)$. Since the sum of strongly convex functions is a strongly convex function, $f(x)$, which sums $g(x_i)$ from $i=1$ to $i=D$, is also $2$-strongly convex in $N(x^\star; d_m\delta_m)$. Therefore, the second condition of new $\sigma$-niceness holds and Rastrigin's function is a new $2$-nice function.
\end{proof}

\section{B. Detail of the test functions for explicit graduated optimization}
Let $\bm{x} := (x_1, \cdots, x_D)^\top \in \mathbb{R}^D$. 

\begin{itemize}
\item Ackley's function \cite{Ackley1987Aco, Thomas1993AnO} is defined as follows:
\begin{align*}
f(\bm{x}) := -20\exp\left( -0.2\sqrt{\frac{1}{D} \sum_{i=1}^{D} x_i^2} \right) - \exp\left( \frac{1}{D} \sum_{i=1}^{D} \cos(2\pi x_i)\right) + e + 20.
\end{align*}

\item Alpine1 function \cite{Rahnamayan2007Ano} is defines as follows:
\begin{align*}
f(\bm{x}) := \sum_{i=1}^{D} |x_i \sin(x_i) + 0.1 x_i|.
\end{align*}

\item Drop-Wave function \cite{Molga2005Tes} is defined as follows:
\begin{align*}
f(\bm{x}) := 1 - \frac{1+\cos\left( 12\sqrt{\sum_{i=1}^{D} x_i^2}\right)}{0.5 \sum_{i=1}^{D} x_i^2 + 2}.
\end{align*}

\item Ellipsoid function (weighted sphere function or hyper-ellipsoid function) \cite{Yang2010Eng} is defined as follows:
\begin{align*}
f(\bm{x}) := \sum_{i=1}^{D} i x_i^2.
\end{align*}

\item Griewank function \cite{Griewank1981Gen} is defined as follows:
\begin{align*}
f(\bm{x}) := \frac{1}{4000}\sum_{i=1}^{D}x_i^2 - \prod_{i=1}^{D} \cos\left( \frac{x_i}{\sqrt{i}}\right) + 1.
\end{align*}

\item HappyCat function \cite{Beyer2012Hap} is defined as follows:
\begin{align*}
f(\bm{x}) := \left| \sum_{i=1}^{D} x_i^2 - D \right|^{0.25} + \frac{0.5 \sum_{i=1}^{D} x_i^2 + \sum_{i=1}^{D} x_i}{D} + 0.5.
\end{align*}

\item HGBat function \cite{Ying2016GPU} is defined as follows:
\begin{align*}
f(\bm{x}) := \left| \left( \sum_{i=1}^{D} x_i^2 \right)^2 - \left( \sum_{i=1}^{D} x_i \right)^2 \right|^{0.5} + \frac{0.5 \sum_{i=1}^{D} x_i^2 + \sum_{i=1}^{D} x_i}{D} + 0.5.
\end{align*}

\item Modified Ridge function \cite{Plevris2022ACo} is defined as follows:
\begin{align*}
f(\bm{x}) := |x_1| + 2\left( \sum_{i=2}^{D} x_i^2 \right)^{0.1}.
\end{align*}

\item Rastrigin's function \cite{Torn1989Glo, Rudolph1990Glo} is defined as follows:
\begin{align*}
f(\bm{x}) := \sum_{i=1}^{D} \left( x_i^2 - 10\cos(2\pi x_i) \right) + 10D.
\end{align*}

\item Rosenbrock's function \cite{Rosenbrock1960AnA, Dixon1994Eff} is defined as follows:
\begin{align*}
f(\bm{x}) := \sum_{i=1}^{D-1} \left( 100\left( x_{i+1} - x_i^2 \right)^2 + (x_i-1)^2 \right).
\end{align*}

\item Rotated Hyper-ellipsoid function \cite{Molga2005Tes} is defined as follows:
\begin{align*}
f(\bm{x}) := \sum_{i=1}^{D}(D+1-i)x_i^2.
\end{align*}

\item Salomon function \cite{Salomon1996Re-} is defined as follows:
\begin{align*}
f(\bm{x}) := 1-\cos \left( 2\pi \sqrt{\sum_{i=1}^{D} x_i^2} \right) + 0.1 \sqrt{\sum_{i=1}^{D} x_i^2}.
\end{align*}

\item Schaffer's F7 function \cite{Schaffer1985Mul, Caruana1989Ast} is defined as follows:
\begin{align*}
f(\bm{x}) := \left( \frac{1}{D-1} \sum_{i=1}^{D-1} \left( \left( x_i^2 + x_{i+1}^2 \right)^{1/4} + \left( x_i^2 + x_{i+1}^2 \right)^{1/4} \sin^2 \left( 50\left( x_i^2 + x_{i+1}^2 \right)^{1/10} \right)\right) \right)^2.
\end{align*}

\item Schwefel function \cite{Schwefel1981Num} is defined as follows:
\begin{align*}
f(\bm{x}) := 418.9829D - \sum_{i=1}^{D} x_i \sin\left( \sqrt{\left| x_i \right|}\right).
\end{align*}

\item Schwefel 2.21 function \cite{Schwefel1981Num} is defined as follows:
\begin{align*}
f(\bm{x}) := \max_{i=1,\cdots,D} |x_i|.
\end{align*}

\item Sphere function \cite{Schumer1968Ada} is defined as follows:
\begin{align*}
f(\bm{x}) := \sum_{i=1}^{D} x_i^2.
\end{align*}
\end{itemize}

\begin{table*}[htbp]
    \centering
    \begin{tabular}{c|cccc}
        \hline
        function & search range & optimal solution &Algorithm \ref{alg:sgd}'s learning rate\\
        \hline
        Ackley's & $[-32.768, 32.768]^D$ & $\bm{x}^\star = (0, 0, \cdots, 0)$ & $5\delta_m$ \\
        Alpine1 & $[-10, 10]^D$ & $\bm{x}^\star = (0, 0, \cdots, 0)$ & $\delta_m$ \\
        Drop-Wave & $[-5.12, 5.12]^D$ & $\bm{x}^\star = (0, 0, \cdots, 0)$ & $0.1\delta_m$ \\
        Ellipsoid & $[-100, 100]^D$ & $\bm{x}^\star = (0, 0, \cdots, 0)$ & $0.01\delta_m$  \\
        Griewank & $[-100, 100]^D$ & $\bm{x}^\star = (0, 0, \cdots, 0)$ & $(50\delta_m)^{\delta_m}$  \\
        HappyCat & $[-20, 20]^D$ & $\bm{x}^\star = (-1, -1, \cdots, -1)$ & $(10\delta_m)^{\delta_m}$  \\
        HGBat & $[-15, 15]^D$ & $\bm{x}^\star = (-1, -1, \cdots, -1)$ & $0.1\delta_m$  \\
        Modified Ridge & $[-100, 100]^D$ & $\bm{x}^\star = (0, 0, \cdots, 0)$ & $\delta_m$ \\
        Rastrigin's & $[-5.12, 5.12]^D$ & $\bm{x}^\star = (0, 0, \cdots, 0)$ & $0.01\delta_m$ \\
        Rosenbrock's & $[-10, 10]^D$ & $\bm{x}^\star = (0, 0, \cdots, 0)$ & $0.00005\delta_m$ \\
        Rotated Hyper-ellipsoid & $[-100, 100]^D$ & $\bm{x}^\star = (0, 0, \cdots, 0)$ & $0.01\delta_m$ \\
        Salomon & $[-20, 20]^D$ & $\bm{x}^\star = (0, 0, \cdots, 0)$ & $(10\delta_m)^{\delta_m}$ \\
        Schaffer's F7 & $[-100, 100]^D$ & $\bm{x}^\star = (0, 0, \cdots, 0)$ & $20\delta_m$ \\
        Schwefel & $[-500, 500]^D$ & $\bm{x}^\star = (0, 0, \cdots, 0)$ & $10\delta_m$ \\
        Schwefel 2.21 & $[-100, 100]^D$ & $\bm{x}^\star = (0, 0, \cdots, 0)$ & $(10\delta_m)^{\delta_m}$ \\
        Sphere & $[-100, 100]^D$ & $\bm{x}^\star = (0, 0, \cdots, 0)$ & $\delta_m$ \\
        \hline
    \end{tabular}
    \caption{Test function's search range, optimal solution, and used learning rate in Algorithm \ref{alg:sgd}.}
    \label{tab:11}
\end{table*}

\section{C. Convergence analysis of Algorithm \ref{alg:gnc3}}
This section provides a proof of Theorem \ref{thm:02}. The new $\sigma$-nice function ensures that each of the smoothed $M$ functions is $\sigma$-strongly convex in a neighborhood $N(\bm{x}^\star ; d_m |\delta_m|) \ (m \in [M])$ (see Definition \ref{dfn:3.1}). Note that Algorithm \ref{alg:sgd} can therefore be applied to $\sigma$-strongly convex functions, although it does not require the original function $f$ to be $\sigma$-strongly convex. First, we prove the following theorem on the basis of a convergence analysis of gradient descent (Algorithm \ref{alg:sgd}) for $\sigma$-strongly convex functions.

\subsection{Theorem and lemmas for the analyses}
\begin{thm}
[Convergence analysis of Algorithm \ref{alg:sgd}]\citep[Theorem 3]{Sato2023Usi}\label{thm:d.1} Suppose that Assumption (A2) holds, where $\mathsf{G}_{\xi_t}$ is the stochastic gradient of a $\sigma$-strongly convex and $L_g$-smooth function $F \colon \mathbb{R}^d \to \mathbb{R}$, and $\eta < \min \left\{\frac{1}{\sigma}, \frac{2}{L_g} \right\}$. Then, the sequence $(\hat{\bm{x}}_t)_{t \in \mathbb{N}}$ generated by Algorithm \ref{alg:sgd} satisfies

\begin{align*}
\min_{t \in [T]} \mathbb{E} \left[ F \left(\hat{\bm{x}}_t \right) - F(\bm{x}^\star) \right] 
\leq \frac{H_4}{T} + H_3 \eta
= \mathcal{O}\left( \frac{1}{T} + \eta \right),
\end{align*}
where $\bm{x}^\star$ is the global minimizer of $F$, and $H_4 \text{ and } H_3 > 0$ are nonnegative constants.
\end{thm}

Theorem \ref{thm:d.1} shows that Algorithm \ref{alg:sgd} can reach an $\epsilon_m$-neighborhood of optimal solution $\bm{x}_{{\delta}_m}^\star$ of $\hat{f}_{\delta_m}$ in approximately $T_F := \frac{H_4}{\epsilon_m - H_3 \eta}$ iterations.

\begin{lem}\label{lem:d.1}
Suppose that (C2) holds; then $\hat{f}_\delta$ is an $L_f$-Lipschitz function; i.e., for all $\bm{x}, \bm{y} \in \mathbb{R}^d$,
\begin{align*}
\left|\hat{f}_\delta(\bm{x}) - \hat{f}_\delta(\bm{y}) \right| 
\leq L_f \| \bm{x} - \bm{y} \|.
\end{align*}
\begin{proof}
From Definition \ref{dfn:fhat} and (C2), we obtain, for all $\bm{x}, \bm{y} \in \mathbb{R}^d$,
\begin{align*}
\left|\hat{f}_\delta(\bm{x}) - \hat{f}_\delta(\bm{y}) \right|
&=\left| \mathbb{E}_{\bm{u}}\left[f(\bm{x}-\delta \bm{u})\right] - \mathbb{E}_{\bm{u}}\left[f(\bm{y}-\delta \bm{u})\right] \right| \\
&=\left| \mathbb{E}_{\bm{u}} \left[f(\bm{x} - \delta \bm{u}) - f(\bm{y} - \delta \bm{u})\right] \right| \\
&\leq \mathbb{E}_{\bm{u}} \left[\left| f(\bm{x} - \delta \bm{u}) - f(\bm{y} - \delta \bm{u}) \right|\right] \\
&\leq \mathbb{E}_{\bm{u}} \left[L_f \| (\bm{x} - \delta \bm{u}) - (\bm{y} - \delta \bm{u}) \|\right] \\
&= \mathbb{E}_{\bm{u}} \left[L_f \left\| \bm{x} - \bm{y} \right\| \right] \\
&= L_f \| \bm{x} - \bm{y} \|.
\end{align*}
This completes the proof.
\end{proof}
\end{lem}

Lemma \ref{lem:d.1} implies that the Lipschitz constant $L_f$ of the original function $f$ is carried over to the function $\hat{f}_{\delta}$ smoothed by any $\delta \in \mathbb{R}$.

\begin{lem}\label{lem:d.2}
Let $\hat{f}_\delta$ be the smoothed version of $f$; then, for all $\bm{x} \in \mathbb{R}^d$,
\begin{align*}
\left|\hat{f}_\delta (\bm{x}) - f(\bm{x})\right|
\leq |\delta| L_f.
\end{align*}
\begin{proof}
From Definition \ref{dfn:fhat} and (C2), we have, for all $\bm{x}, \bm{y} \in \mathbb{R}^d$,
\begin{align*}
\left|\hat{f}_\delta(\bm{x}) - f(\bm{x}) \right|
&=\left| \mathbb{E}_{\bm{u}}\left[f(\bm{x}-\delta \bm{u})\right] - f(\bm{x}) \right| \\
&=\left| \mathbb{E}_{\bm{u}} \left[f(\bm{x} - \delta \bm{u}) - f(\bm{x}) \right] \right| \\
&\leq \mathbb{E}_{\bm{u}} \left[\left| f(\bm{x} - \delta \bm{u}) - f(\bm{x}) \right|\right] \\
&\leq \mathbb{E}_{\bm{u}} \left[L_f \| (\bm{x} - \delta \bm{u}) - \bm{x} \|\right] \\
&= \mathbb{E}_{\bm{u}} \left[L_f |\delta| \| \bm{u} \| \right] \\
&= |\delta| L_f,
\end{align*}
where $\| \bm{u} \| \leq 1$. This completes the proof.
\end{proof}
\end{lem}

\subsection{Proof of Theorem \ref{thm:02}}
The following proof framework is based on previous studies \cite{Sato2023Usi, Elad2016OnG}.
\begin{proof}
From the definition of $\delta_{m+1}^{\text{\rm{SHB}}}$ (see Algorithm \ref{alg:gnc3}), we obtain
\begin{align*}
\delta_{m+1}^{\text{\rm{SHB}}} &:= \eta_{m+1} \sqrt{(1 + \hat{\beta}_{m+1})\frac{C_{\text{\rm{SHB}}}^2}{b_{m+1}} + \hat{\beta}_{m+1} K_{\text{\rm{SHB}}}^2} \\
&= \kappa_m \eta_m \sqrt{(1 + \rho_m \hat{\beta}_m) \frac{C_{\text{\rm{SHB}}}^2}{\lambda_m b_m} + \rho_m \hat{\beta}_m K_{\text{\rm{SHB}}}^2} \\
&= \gamma_m \delta_m^{\text{\rm{SHB}}}.
\end{align*}
Hence, from $M^p := \frac{1}{\alpha_0 \epsilon}$ and $\gamma_m := \frac{(M-m)^p}{\left\{M-(m-1) \right\}^p}$, 
\begin{align*}
\delta_M^{\text{\rm{SHB}}} &= \delta_1^{\text{\rm{SHB}}} (\gamma_1 \gamma_2 \cdot \gamma_{M-1}) \\
&=\delta_1^{\text{\rm{SHB}}} \cdot \frac{(M-1)^p}{M^p} \cdot \frac{(M-2)^p}{(M-1)^p} \cdot \frac{(M-3)^p}{(M-2)^p} \cdots \frac{1}{2^p} \\
&=\delta_1^{\text{\rm{SHB}}} \cdot \frac{1}{M^p} \\
&=\delta_1^{\text{\rm{SHB}}} \alpha_0 \epsilon.
\end{align*}
According to Theorem \ref{thm:d.1},
\begin{align*}
\mathbb{E}\left[ \hat{f}_{\delta_M^{\text{\rm{SHB}}}} (\bm{x}_{M+1}) - \hat{f}_{\delta_M^{\text{\rm{SHB}}}} (\bm{x}_{\delta_M^{\text{\rm{SHB}}}}^\star)\right] 
\leq \epsilon_M 
:=\sigma^2 {\delta_M^{\text{\rm{SHB}}}}^2
=\left(\sigma \delta_1^{\text{\rm{SHB}}} \alpha_0 \epsilon \right)^2.
\end{align*}
From Lemmas \ref{lem:d.1} and \ref{lem:d.2},
\begin{align*}
f(\bm{x}_{M+2}) - f(\bm{x}^\star)
&=\left\{f(\bm{x}_{M+2}) - \hat{f}_{\delta_M^{\text{\rm{SHB}}}}(\bm{x}_{M+2})\right\} + \left\{\hat{f}_{\delta_M^{\text{\rm{SHB}}}}(\bm{x}^\star) - f(\bm{x}^\star)\right\} + \left\{\hat{f}_{\delta_M^{\text{\rm{SHB}}}}(\bm{x}_{M+2}) - \hat{f}_{\delta_M^{\text{\rm{SHB}}}}(\bm{x}^\star) \right\} \\
&\leq \left\{f(\bm{x}_{M+2}) - \hat{f}_{\delta_M^{\text{\rm{SHB}}}}(\bm{x}_{M+2})\right\} + \left\{\hat{f}_{\delta_M^{\text{\rm{SHB}}}}(\bm{x}^\star) - f(\bm{x}^\star)\right\} + \left\{\hat{f}_{\delta_M^{\text{\rm{SHB}}}}(\bm{x}_{M+2}) - \hat{f}_{\delta_M^{\text{\rm{SHB}}}}(\bm{x}_{\delta_M^{\text{\rm{SHB}}}}^\star)\right\} \\
&\leq \left| \delta_M^{\text{\rm{SHB}}} \right| L_f + \left| \delta_M^{\text{\rm{SHB}}} \right| L_f + \left\{\hat{f}_{\delta_M^{\text{\rm{SHB}}}}(\bm{x}_{M+2}) - \hat{f}_{\delta_M^{\text{\rm{SHB}}}}(\bm{x}_{\delta_M^{\text{\rm{SHB}}}}^\star)\right\} \\
&= 2 \left| \delta_M^{\text{\rm{SHB}}} \right| L_f + \left\{ \hat{f}_{\delta_M^{\text{\rm{SHB}}}} (\bm{x}_{M+2}) - \hat{f}_{\delta_M^{\text{\rm{SHB}}}} (\bm{x}_{M+1})\right\} + \left\{\hat{f}_{\delta_M^{\text{\rm{SHB}}}} (\bm{x}_{M+1}) - \hat{f}_{\delta_M^{\text{\rm{SHB}}}}(\bm{x}_{\delta_M^{\text{\rm{SHB}}}}^\star) \right\} \\
&\leq 2 \left| \delta_M^{\text{\rm{SHB}}} \right| L_f + L_f \left\| \bm{x}_{M+2} - \bm{x}_{M+1} \right\| + \left\{\hat{f}_{\delta_M^{\text{\rm{SHB}}}} (\bm{x}_{M+1}) - \hat{f}_{\delta_M^{\text{\rm{SHB}}}}(\bm{x}_{\delta_M^{\text{\rm{SHB}}}}^\star) \right\}. 
\end{align*}
Then, we have
\begin{align*}
\mathbb{E} \left[ f(\bm{x}_{M+2}) - f(\bm{x}^\star) \right]
&\leq 2 \left| \delta_M^{\text{\rm{SHB}}} \right| L_f + 2L_f d_M \left| \delta_M^{\text{\rm{SHB}}} \right| + \epsilon_M \\
&\leq 2 \left| \delta_M^{\text{\rm{SHB}}} \right| L_f + 2L_f \bar{d} \left| \delta_M^{\text{\rm{SHB}}} \right| + \epsilon_M \\
&=2L_f \left| \delta_M^{\text{\rm{SHB}}} \right| \left(1+ \bar{d}\right) + \epsilon_M,
\end{align*}
where $\left\| \bm{x}_{M+2} - \bm{x}_{M+1} \right\| \leq 2 d_M \left| \delta_M^{\text{\rm{SHB}}} \right|$, since $\bm{x}_{M+2}, \bm{x}_{M+1} \in N(\bm{x}^\star; d_M \left|\delta_M^{\text{\rm{SHB}}} \right|)$, and $d_M \leq \bar{d} < +\infty$. Therefore, 
\begin{align*}
\mathbb{E} \left[ f(\bm{x}_{M+2}) - f(\bm{x}^\star) \right]
&\leq 2L_f\left(1+\bar{d}\right)\left| \delta_1^{\text{\rm{SHB}}} \right| \alpha_0 \epsilon + \left(\sigma \left| \delta_1^{\text{\rm{SHB}}} \right| \alpha_0 \epsilon \right)^2 \\
&\leq \epsilon,
\end{align*}
where $\alpha_0 := \min \left\{ \frac{1}{4L_f \left| \delta_1^{\text{\rm{SHB}}} \right| \left(1+\bar{d}\right)}, \frac{1}{\sqrt{2}\sigma \left| \delta_1^{\text{\rm{SHB}}}\right|} \right\}$.\\
Let $T_{\text{total}}$ be the total number of queries made by Algorithm \ref{alg:gnc3}; then, 
\begin{align*}
T_{\text{total}} 
&= \sum_{m=1}^{M+1} \frac{H_4}{\epsilon_m - H_3 \eta_m} 
= \sum_{m=1}^{M+1} \frac{H_4}{\sigma^2{\delta_m^{\text{\rm{SHB}}}}^2 - H_3\eta_m} \\
&= \frac{H_4}{\sigma^2{\delta_1^{\text{\rm{SHB}}}}^2 - H_3\eta_1} + \frac{H_4}{\sigma^2\delta_2^2 - H_3\eta_2} + \cdots + \frac{2H_4}{\sigma^2{\delta_M^{\text{\rm{SHB}}}}^2 - H_3\eta_{M}}\\
&\leq \frac{H_4(M+1)}{\sigma^2 {\delta_M^{\text{\rm{SHB}}}}^2 - H_3\eta_M} \\
&\leq \frac{H_4(M+1)}{\sigma^2 {\delta_M^{\text{\rm{SHB}}}}^2 - H_3\eta_1} \\
&= \frac{H_4(M+1)}{\sigma^2 {\delta_1^{\text{\rm{SHB}}}}^2 \frac{1}{M^{2p}} - H_3\eta_1} \\
&= \frac{H_4M^{2p}(M+1)}{\sigma^2 {\delta_1^{\text{\rm{SHB}}}}^2 - H_3\eta_1 M^{2p}} \\
&= \frac{H_4M^{2p+1}}{\sigma^2 {\delta_1^{\text{\rm{SHB}}}}^2 - H_3\eta_1 M^{2p}} + \frac{H_4M^{2p}}{\sigma^2 {\delta_1^{\text{\rm{SHB}}}}^2 - H_3\eta_1 M^{2p}}.
\end{align*}
From $M^p := \frac{1}{\alpha_0 \epsilon}$, 
\begin{align*}
T_{\text{total}} 
&= \frac{H_4 \left( \frac{1}{\alpha_0 \epsilon} \right)^{\frac{1}{p} + 2}}{\sigma^2 {\delta_1^{\text{\rm{SHB}}}}^2 - H_3\eta_1 \left( \frac{1}{\alpha_0 \epsilon} \right)^2} + \frac{H_4 \left( \frac{1}{\alpha_0 \epsilon} \right)^2}{\sigma^2 {\delta_1^{\text{\rm{SHB}}}}^2 - H_3\eta_1 \left( \frac{1}{\alpha_0 \epsilon} \right)^2} \\
&\leq \frac{2H_4 \left( \frac{1}{\alpha_0 \epsilon} \right)^{\frac{1}{p} + 2}}{\sigma^2 {\delta_1^{\text{\rm{SHB}}}}^2 - H_3\eta_1 \left( \frac{1}{\alpha_0 \epsilon} \right)^2} \\
&= \frac{2H_4 \left( \frac{1}{\alpha_0 \epsilon} \right)^{\frac{1}{p}}}{\sigma^2 {\delta_1^{\text{\rm{SHB}}}}^2 (\alpha_0 \epsilon)^2 - H_3\eta_1} \\
&= \frac{2H_4}{\sigma^2 {\delta_1^{\text{\rm{SHB}}}}^2 (\alpha_0 \epsilon)^{\frac{1}{p}+2} - H_3\eta_1 (\alpha_0 \epsilon)^{\frac{1}{p}}} \\
&\leq \frac{2H_4}{\sigma^2 {\delta_1^{\text{\rm{SHB}}}}^2 (\alpha_0 \epsilon)^{\frac{1}{p}} - H_3\eta_1 (\alpha_0 \epsilon)^{\frac{1}{p}}} \\
&= \mathcal{O}\left( \frac{1}{\epsilon^{\frac{1}{p}}}\right).
\end{align*}
This completes the proof.
\end{proof}

\section{D. Experiments on implicit graduated optimization with SHB}
\label{sec:f.2}
This section supplements the hyperparameter update rules in Figure \ref{fig:04} and provides the results for WideResNet-28-10 with Algorithm \ref{alg:gnc3} (Figure \ref{fig:08}). 
Recall that $\eta_{m+1} := \kappa_m \eta_m, b_{m+1} := \lambda_m b_m, \hat{\beta}_{m+1} := \rho_m \hat{\beta}_m$, and 
\begin{align*}
\frac{\kappa_m \eta_m \sqrt{(1 + \rho_m \hat{\beta}_m) \frac{C_{\text{\rm{SHB}}}^2}{\lambda_m b_m} + \rho_m \hat{\beta}_m K_{\text{\rm{SHB}}}^2}}{\eta_m \sqrt{(1 + \hat{\beta}_m) \frac{C_{\text{\rm{SHB}}}^2}{b_m} + \hat{\beta}_m K_{\text{\rm{SHB}}}^2}} = \gamma_m := \frac{(M-m)^{0.9}}{\left\{ M-(m-1)\right\}^{0.9}},
\end{align*}
where $M = 200, m \in [M], \kappa_m, \rho_m \in (0,1], \lambda_m \geq 1$, and $\delta_{m+1}^{\text{\rm{SHB}}} := \gamma_m \delta_{m}^{\text{\rm{SHB}}}$ is the degree of smoothing at epoch $m+1$.

The "only lr decayed" method reduces only the learning rate, so the decay rate of noise level $\delta_{m+1}^{\text{\rm{SHB}}} / \delta_m^{\text{\rm{SHB}}}$ is equal to $\gamma_m$. Hence, we should use $\kappa_m := \gamma_m, \lambda_m := 1$, and $\rho_m := 1$ ; i.e., the "only lr decayed" method means a polynomial decay rate of $p = 0.9$. The "only momentum decayed" method reduces only the momentum, so the decay rate of noise level $\delta_{m+1}^{\text{\rm{SHB}}} / \delta_m^{\text{\rm{SHB}}}$ is also equal to $\gamma_m$. Hence, we should use $\kappa_m := 1, \lambda_m := 1$, and $\rho_m$ such that
\begin{align*}
\frac{\sqrt{(1 + \rho_m \hat{\beta}_m) \frac{C_{\text{\rm{SHB}}}^2}{b} + \rho_m \hat{\beta}_m K_{\text{\rm{SHB}}^2}}}{\sqrt{(1 + \hat{\beta}_m) \frac{C_{\text{\rm{SHB}}}^2}{b} + \hat{\beta}_m K_{\text{\rm{SHB}}}^2}} = \gamma_m \ \text{ i.e., }\ 
\rho_m := \frac{\gamma_m^2\left\{ \frac{C_{\text{\rm{SHB}}}^2}{b} + \hat{\beta}_m \left( \frac{C_{\text{\rm{SHB}}}^2}{b} + K_{\text{\rm{SHB}}}^2 \right)\right\} - \frac{C_{\text{\rm{SHB}}}^2}{b}}{\hat{\beta}_m \left( \frac{C_{\text{\rm{SHB}}}^2}{b} + K_{\text{\rm{SHB}}}^2 \right)}.
\end{align*}
Note that, since $\hat{\beta}_{m+1} := \frac{\beta_{m+1}(\beta_{m+1}^2 - \beta_{m+1} + 1)}{(1-\beta_{m+1})^2}$, to find momentum $\beta_{m+1}$, we need to solve the cubic equation for $\beta_{m+1}$. The "lr decayed and batch size increased" method reduces the learning rate and increases the batch size, so the decay rate of the noise level $\delta_{m+1}^{\text{\rm{SHB}}} / \delta_m^{\text{\rm{SHB}}}$ is also equal to $\gamma_m$. Hence, we should use $\rho_m := 1$, $\lambda_m := \frac{(M-m)^{0.747}}{\left\{ M-(m-1)\right\}^{0.747}}$, and $\kappa_m$ such that
\begin{align*}
\kappa_m := \frac{\gamma_m \sqrt{(1 + \rho_m \hat{\beta}_m) \frac{C_{\text{\rm{SHB}}}^2}{\lambda_m b_m} + \rho_m \hat{\beta}_m K_{\text{\rm{SHB}}^2}}}{\sqrt{(1 + \hat{\beta}_m) \frac{C_{\text{\rm{SHB}}}^2}{b_m} + \hat{\beta}_m K_{\text{\rm{SHB}}}^2}},
\end{align*}
where $\lambda_m$ was tuned so that when the initial batch size was $256$, increasing the batch size was still computationally small enough to run on our machines when 200 epochs were reached. For the remaining methods, the parameters were similarly updated so that the decay rate of the degree of smoothing was $\gamma_m := \frac{(M-m)^{0.9}}{\left\{ M-(m-1)\right\}^{0.9}}$.

\begin{figure}
\begin{minipage}[t]{0.49\textwidth}
\centering
\includegraphics[width=1\textwidth]{2.acc-loss-shb-crop.pdf}%
\caption{Accuracy score for testing and loss function value for training versus epochs in training of \textbf{ResNet18} on CIFAR100 dataset with implicit graduated optimization with \textbf{SHB} (Algorithm \ref{alg:gnc3}). \textbf{This is the same graph shown in Figure \ref{fig:04}}.}
\label{fig:imp-shb}
\end{minipage}%
\hspace{0.01\textwidth}
\begin{minipage}[t]{0.49\textwidth}
\centering
\includegraphics[width=1\textwidth]{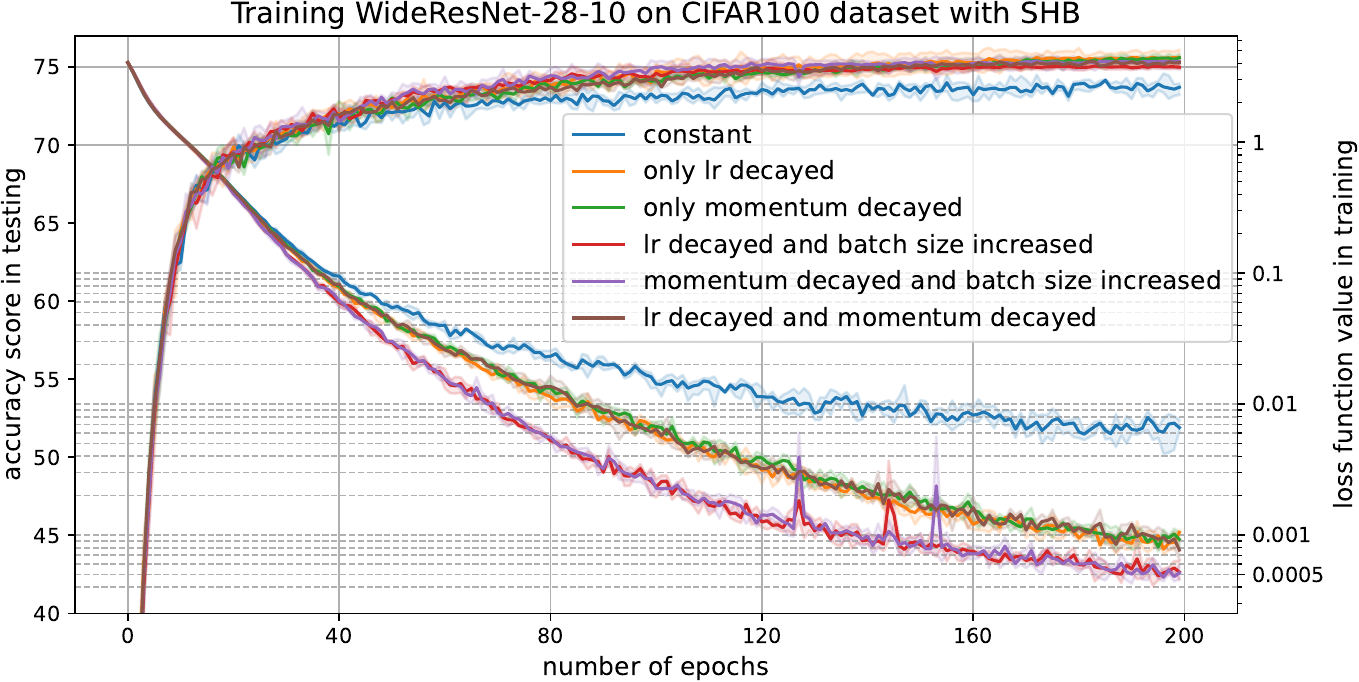}%%
\caption{Accuracy score for testing and loss function value for training versus epochs in training of \textbf{WideResNet-28-10} on CIFAR100 dataset with implicit graduated optimization with \textbf{SHB} (Algorithm \ref{alg:gnc3}).}
\label{fig:08}
\end{minipage}
\end{figure}

\begin{figure}[htbp]
\begin{minipage}[t]{0.99\hsize}
\centering
\includegraphics[width=1\textwidth]{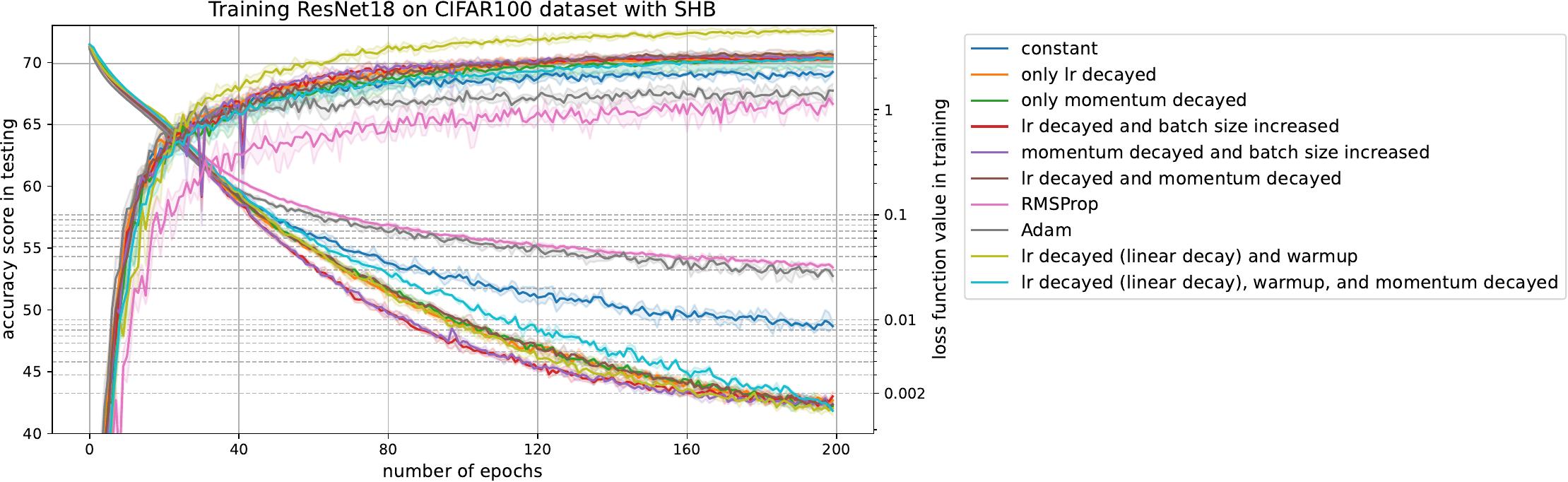}%
\caption{Complete version of Figures \ref{fig:04} and \ref{fig:imp-shb}.}
\label{fig:999}
\end{minipage}
\end{figure}

\section{E. Implicit graduated optimization with NSHB}
\label{sec:E}
This section provides the NSHB version of Algorithm \ref{alg:gnc4} and its convergence analysis.

\begin{algorithm}
\caption{Implicit Graduated Optimization with NSHB}
\label{alg:gnc4}
\begin{algorithmic}
\REQUIRE$\epsilon >0, p \in (0,1], \bar{d} > 0, \bm{x}_1 \in \mathbb{R}^d, \eta_1 > 0 , b_1 \in \mathbb{N}, \beta_1 \in [0,1), C_{\text{\rm{NSHB}}}^2 > 0$
\STATE{$\delta_1^{\text{\rm{NSHB}}} := \eta_1 \sqrt{\frac{1}{1-\beta_1} \frac{C_{\text{NSHB}}^2}{b}}$}
\STATE{$\alpha_0 := \min \left\{ \frac{1}{4L_f \left| \delta_1^{\text{\rm{NSHB}}} \right| \left(1+\bar{d}\right)}, \frac{1}{\sqrt{2}\sigma \left| \delta_1^{\text{\rm{NSHB}}} \right|} \right\} $}
\STATE{$M^p:= \frac{1}{\alpha_0 \epsilon}$}
\FOR{$m=1$ to $M+1$}
\IF{$m \neq M+1$}
\STATE{$\epsilon_m := \sigma^2 {\delta_m^{\text{\rm{NSHB}}}}^2, $}
\STATE{$T_F := H_4 / \left( \epsilon_m - H_3 \eta_m \right)$}
\STATE{$\gamma_m := \frac{(M-m)^p}{\left\{ M-(m-1)\right\}^p}$}
\STATE{$\frac{\kappa_m \eta_m \sqrt{\frac{1}{1-\rho_m\beta_m} \frac{C_{\text{NSHB}}^2}{\lambda_m b_m}}}{\eta_m \sqrt{\frac{1}{1-\beta_m} \frac{C_{\text{NSHB}}^2}{b_m}}} = \gamma_m$ $(\kappa_m, \rho_m \in (0,1], \lambda_m \geq 1)$}
\ENDIF
\STATE{$\bm{x}_{m+1} := \text{SGD}(T_F, \bm{x}_m, \hat{f}_{\delta_m}, \eta_m, b_m)$}
\STATE{$\eta_{m+1} := \kappa_m \eta_m, b_{m+1} := \lambda_m b_m, \beta_{m+1}^2 := \rho_m \beta_m^2 $}
\STATE{$\delta_{m+1}^{\text{\rm{NSHB}}} := \eta_{m+1} \sqrt{\frac{1}{1-\beta_{m+1}} \frac{C_{\text{NSHB}}^2}{b_{m+1}}}$}
\ENDFOR
\RETURN $\bm{x}_{M +2}$
\end{algorithmic}
\end{algorithm}

\begin{thm}
[Convergence analysis of Algorithm \ref{alg:gnc4}]\label{thm:E.1} Let $\epsilon \in (0, 1]$ and $f$ be an $L_f$-Lipschitz new $\sigma$-nice function. Suppose that we apply Algorithm \ref{alg:gnc4}; after $\mathcal{O}\left( 1/\epsilon^{\frac{1}{p}}\right)$ rounds. Then, the algorithm reaches an $\epsilon$-neighborhood of the global optimal solution $\bm{x}^\star$.
\end{thm}
\begin{proof}
From the definition of $\delta_{m+1}^{\text{\rm{NSHB}}}$ (see Algorithm \ref{alg:gnc4}), we obtain
\begin{align*}
\delta_{m+1}^{\text{\rm{NSHB}}} &:= \eta_{m+1} \sqrt{\frac{1}{1-\beta_{m+1}} \frac{C_{\text{NSHB}}^2}{b_{m+1}}} \\
&= \kappa_m \eta_{m} \sqrt{\frac{1}{1-\rho_m\beta_m} \frac{C_{\text{NSHB}}^2}{\lambda_m b_m}} \\
&= \gamma_m \delta_m^{\text{\rm{NSHB}}}.
\end{align*}
Hence, from $M^p := \frac{1}{\alpha_0 \epsilon}$ and $\gamma_m := \frac{(M-m)^p}{\left\{M-(m-1) \right\}^p}$, 
\begin{align*}
\delta_M^{\text{\rm{NSHB}}} &= \delta_1^{\text{\rm{NSHB}}} (\gamma_1 \gamma_2 \cdot \gamma_{M-1}) \\
&=\delta_1^{\text{\rm{NSHB}}} \cdot \frac{(M-1)^p}{M^p} \cdot \frac{(M-2)^p}{(M-1)^p} \cdot \frac{(M-3)^p}{(M-2)^p} \cdots \frac{1}{2^p} \\
&=\delta_1^{\text{\rm{NSHB}}} \cdot \frac{1}{M^p} \\
&=\delta_1^{\text{\rm{NSHB}}} \alpha_0 \epsilon.
\end{align*}
According to Theorem \ref{thm:d.1},
\begin{align*}
\mathbb{E}\left[ \hat{f}_{\delta_M^{\text{\rm{NSHB}}}} (\bm{x}_{M+1}) - \hat{f}_{\delta_M^{\text{\rm{NSHB}}}} (\bm{x}_{\delta_M^{\text{\rm{NSHB}}}}^\star)\right] 
\leq \epsilon_M 
:=\sigma^2 {\delta_M^{\text{\rm{NSHB}}}}^2
=\left(\sigma \delta_1^{\text{\rm{NSHB}}} \alpha_0 \epsilon \right)^2.
\end{align*}
From Lemmas \ref{lem:d.1} and \ref{lem:d.2},
\begin{align*}
f(\bm{x}_{M+2}) - f(\bm{x}^\star)
&=\left\{f(\bm{x}_{M+2}) - \hat{f}_{\delta_M^{\text{\rm{NSHB}}}}(\bm{x}_{M+2})\right\} + \left\{\hat{f}_{\delta_M^{\text{\rm{NSHB}}}}(\bm{x}^\star) - f(\bm{x}^\star)\right\} + \left\{\hat{f}_{\delta_M^{\text{\rm{NSHB}}}}(\bm{x}_{M+2}) - \hat{f}_{\delta_M^{\text{\rm{NSHB}}}}(\bm{x}^\star) \right\} \\
&\leq \left\{f(\bm{x}_{M+2}) - \hat{f}_{\delta_M^{\text{\rm{NSHB}}}}(\bm{x}_{M+2})\right\} + \left\{\hat{f}_{\delta_M^{\text{\rm{NSHB}}}}(\bm{x}^\star) - f(\bm{x}^\star)\right\} + \left\{\hat{f}_{\delta_M^{\text{\rm{NSHB}}}}(\bm{x}_{M+2}) - \hat{f}_{\delta_M^{\text{\rm{NSHB}}}}(\bm{x}_{\delta_M^{\text{\rm{NSHB}}}}^\star)\right\} \\
&\leq \left| \delta_M^{\text{\rm{NSHB}}} \right| L_f + \left| \delta_M^{\text{\rm{NSHB}}} \right| L_f + \left\{\hat{f}_{\delta_M^{\text{\rm{NSHB}}}}(\bm{x}_{M+2}) - \hat{f}_{\delta_M^{\text{\rm{NSHB}}}}(\bm{x}_{\delta_M^{\text{\rm{NSHB}}}}^\star)\right\} \\
&= 2 \left| \delta_M^{\text{\rm{NSHB}}} \right| L_f + \left\{ \hat{f}_{\delta_M^{\text{\rm{NSHB}}}} (\bm{x}_{M+2}) - \hat{f}_{\delta_M^{\text{\rm{NSHB}}}} (\bm{x}_{M+1})\right\} + \left\{\hat{f}_{\delta_M^{\text{\rm{NSHB}}}} (\bm{x}_{M+1}) - \hat{f}_{\delta_M^{\text{\rm{NSHB}}}}(\bm{x}_{\delta_M^{\text{\rm{NSHB}}}}^\star) \right\} \\
&\leq 2 \left| \delta_M^{\text{\rm{NSHB}}} \right| L_f + L_f \left\| \bm{x}_{M+2} - \bm{x}_{M+1} \right\| + \left\{\hat{f}_{\delta_M^{\text{\rm{NSHB}}}} (\bm{x}_{M+1}) - \hat{f}_{\delta_M^{\text{\rm{NSHB}}}}(\bm{x}_{\delta_M^{\text{\rm{NSHB}}}}^\star) \right\}. 
\end{align*}
Then, we have
\begin{align*}
\mathbb{E} \left[ f(\bm{x}_{M+2}) - f(\bm{x}^\star) \right]
&\leq 2 \left| \delta_M^{\text{\rm{NSHB}}} \right| L_f + 2L_f d_M \left| \delta_M^{\text{\rm{NSHB}}} \right| + \epsilon_M \\
&\leq 2 \left| \delta_M^{\text{\rm{NSHB}}} \right| L_f + 2L_f \bar{d} \left| \delta_M^{\text{\rm{NSHB}}} \right| + \epsilon_M \\
&=2L_f \left| \delta_M^{\text{\rm{NSHB}}} \right| \left(1+ \bar{d}\right) + \epsilon_M,
\end{align*}
where $\left\| \bm{x}_{M+2} - \bm{x}_{M+1} \right\| \leq 2 d_M \left| \delta_M^{\text{\rm{NSHB}}} \right|$, since $\bm{x}_{M+2}, \bm{x}_{M+1} \in N(\bm{x}^\star; d_M \left| \delta_M^{\text{\rm{NSHB}}} \right|)$, and $d_M \leq \bar{d} < +\infty$. Therefore, 
\begin{align*}
\mathbb{E} \left[ f(\bm{x}_{M+2}) - f(\bm{x}^\star) \right]
&\leq 2L_f\left(1+\bar{d}\right) \left| \delta_1^{\text{\rm{NSHB}}} \right| \alpha_0 \epsilon + \left(\sigma \left| \delta_1^{\text{\rm{NSHB}}} \right| \alpha_0 \epsilon \right)^2 \\
&\leq \epsilon,
\end{align*}
where $\alpha_0 := \min \left\{ \frac{1}{4L_f \left| \delta_1^{\text{\rm{NSHB}}}\right| \left(1+\bar{d}\right)}, \frac{1}{\sqrt{2}\sigma \left| \delta_1^{\text{\rm{NSHB}}} \right|} \right\}$.\\
Let $T_{\text{total}}$ be the total number of queries made by Algorithm \ref{alg:gnc4}; then, 
\begin{align*}
T_{\text{total}} 
&= \sum_{m=1}^{M+1} \frac{H_4}{\epsilon_m - H_3 \eta_m} 
= \sum_{m=1}^{M+1} \frac{H_4}{\sigma^2{\delta_m^{\text{\rm{NSHB}}}}^2 - H_3\eta_m} \\
&= \frac{H_4}{\sigma^2{\delta_1^{\text{\rm{NSHB}}}}^2 - H_3\eta_1} + \frac{H_4}{\sigma^2\delta_2^2 - H_3\eta_2} + \cdots + \frac{2H_4}{\sigma^2{\delta_M^{\text{\rm{NSHB}}}}^2 - H_3\eta_{M}}\\
&\leq \frac{H_4(M+1)}{\sigma^2 {\delta_M^{\text{\rm{NSHB}}}}^2 - H_3\eta_M} \\
&\leq \frac{H_4(M+1)}{\sigma^2 {\delta_M^{\text{\rm{NSHB}}}}^2 - H_3\eta_1} \\
&= \frac{H_4(M+1)}{\sigma^2 {\delta_1^{\text{\rm{NSHB}}}}^2 \frac{1}{M^{2p}} - H_3\eta_1} \\
&= \frac{H_4M^{2p}(M+1)}{\sigma^2 {\delta_1^{\text{\rm{NSHB}}}}^2 - H_3\eta_1 M^{2p}} \\
&= \frac{H_4M^{2p+1}}{\sigma^2 {\delta_1^{\text{\rm{NSHB}}}}^2 - H_3\eta_1 M^{2p}} + \frac{H_4M^{2p}}{\sigma^2 {\delta_1^{\text{\rm{NSHB}}}}^2 - H_3\eta_1 M^{2p}}.
\end{align*}
From $M^p := \frac{1}{\alpha_0 \epsilon}$, 
\begin{align*}
T_{\text{total}} 
&= \frac{H_4 \left( \frac{1}{\alpha_0 \epsilon} \right)^{\frac{1}{p} + 2}}{\sigma^2 {\delta_1^{\text{\rm{NSHB}}}}^2 - H_3\eta_1 \left( \frac{1}{\alpha_0 \epsilon} \right)^2} + \frac{H_4 \left( \frac{1}{\alpha_0 \epsilon} \right)^2}{\sigma^2 {\delta_1^{\text{\rm{NSHB}}}}^2 - H_3\eta_1 \left( \frac{1}{\alpha_0 \epsilon} \right)^2} \\
&\leq \frac{2H_4 \left( \frac{1}{\alpha_0 \epsilon} \right)^{\frac{1}{p} + 2}}{\sigma^2 {\delta_1^{\text{\rm{NSHB}}}}^2 - H_3\eta_1 \left( \frac{1}{\alpha_0 \epsilon} \right)^2} \\
&= \frac{2H_4 \left( \frac{1}{\alpha_0 \epsilon} \right)^{\frac{1}{p}}}{\sigma^2 {\delta_1^{\text{\rm{NSHB}}}}^2 (\alpha_0 \epsilon)^2 - H_3\eta_1} \\
&= \frac{2H_4}{\sigma^2 {\delta_1^{\text{\rm{NSHB}}}}^2 (\alpha_0 \epsilon)^{\frac{1}{p}+2} - H_3\eta_1 (\alpha_0 \epsilon)^{\frac{1}{p}}} \\
&\leq \frac{2H_4}{\sigma^2 {\delta_1^{\text{\rm{NSHB}}}}^2 (\alpha_0 \epsilon)^{\frac{1}{p}} - H_3\eta_1 (\alpha_0 \epsilon)^{\frac{1}{p}}} \\
&= \mathcal{O}\left( \frac{1}{\epsilon^{\frac{1}{p}}}\right).
\end{align*}
This completes the proof.
\end{proof}

\subsection{Experiments on implicit graduated optimization with NSHB}
\label{sec:f.3}
To test the ability of Algorithm \ref{alg:gnc3} to reduce stochastic noise, we compared  a "constant" method in which the learning rate, batch size, and momentum were all constant with a method in which the hyperparameters were updated to reduce the degree of smoothing $\delta^{\text{\rm{NSHB}}}$. Figures \ref{fig:imp-nshb} and \ref{fig:imp-nshb-wide} plot the accuracy in testing and the loss function value in training with NSHB versus the number of epochs. We tuned the noise reduction methods so that the decay rate was $\gamma_m := \frac{(M-m)^{0.9}}{\left\{ M-(m-1)\right\}^{0.9}} \ (M = 200, m \in [M])$.
Note that results such as ``only momentum decayed'' are omitted because there are no parameters that would satisfy the decay rate $\gamma_m$.

\begin{figure}[htbp]
\begin{minipage}[t]{0.49\columnwidth}
\centering
\includegraphics[width=\columnwidth]{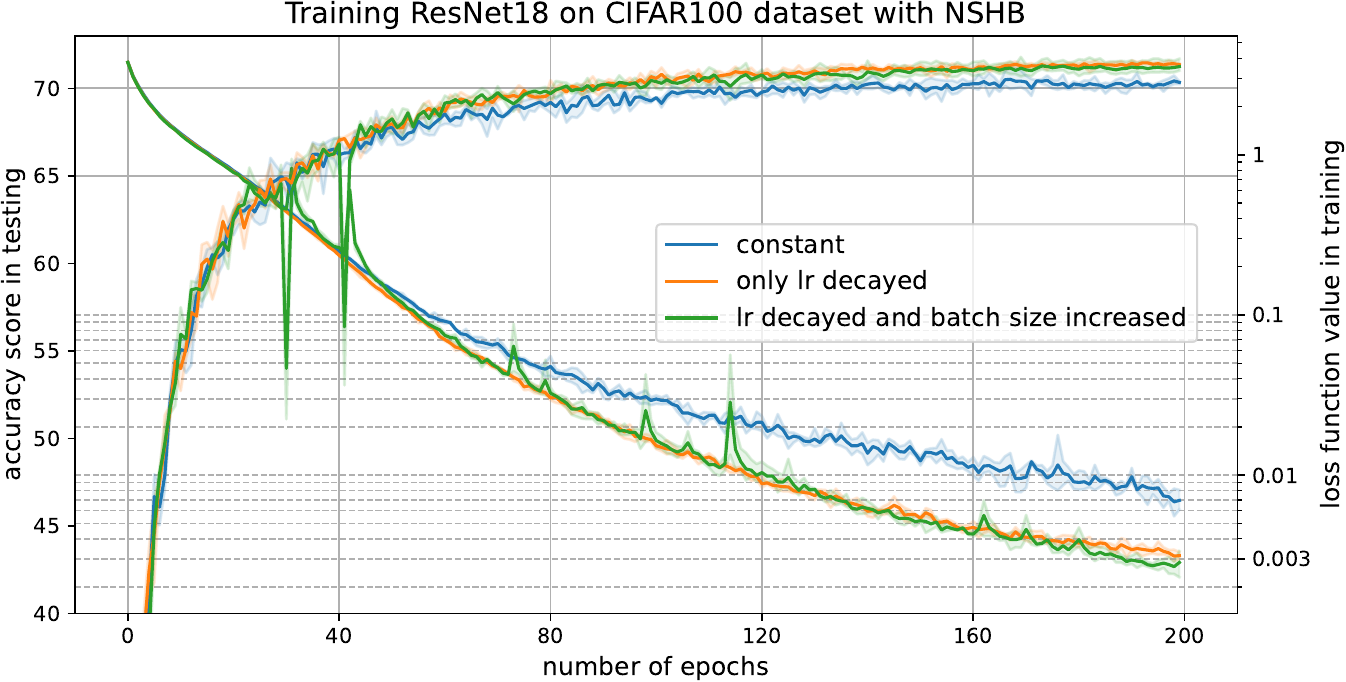}%
\caption{Accuracy score for testing and loss function value for training versus epochs in training of \textbf{ResNet18} on CIFAR100 dataset with implicit graduated optimization with \textbf{NSHB} (Algorithm \ref{alg:gnc4}).}
\label{fig:imp-nshb}
\end{minipage}
\hspace{0.01\columnwidth}
\begin{minipage}[t]{0.49\columnwidth}
\centering
\includegraphics[width=\columnwidth]{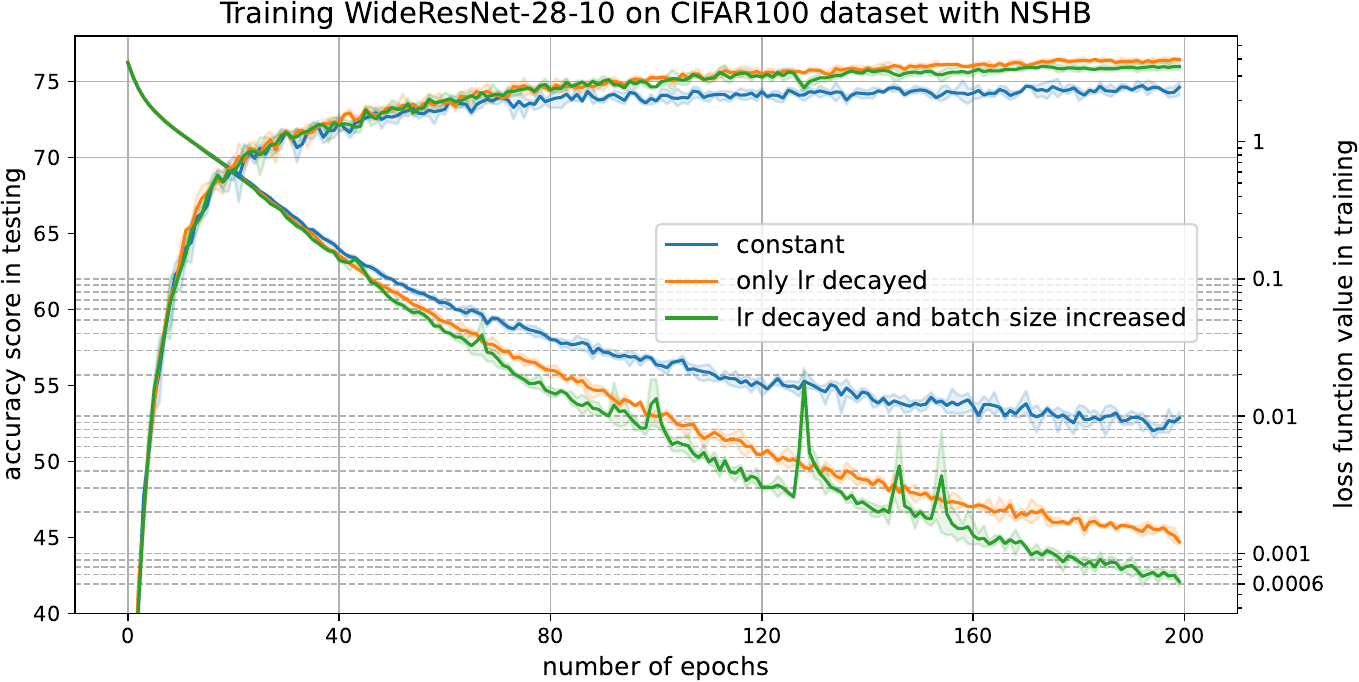}%%
\caption{Accuracy score for testing and loss function value for training versus epochs in training of \textbf{WideResNet-28-10} on CIFAR100 dataset with implicit graduated optimization with \textbf{NSHB} (Algorithm \ref{alg:gnc4}).}
\label{fig:imp-nshb-wide}
\end{minipage}
\end{figure}

\section{F. Experiments on optimal noise scheduling with SHB and NSHB}
\label{sec:f.4}
Sato and Iiduka identified the conditions that the decay rate $(\gamma_m)_{m \in [M]}$ of noise $(\delta_m)_{m \in [M]}$ must satisfy by considering the conditions necessary for the success of a graduated optimization algorithm for the new $\sigma$-nice function \cite{Sato2023Usi}. They found that decay rate $\gamma_m$ should satisfy
\begin{align}\label{eq:18}
\frac{\sqrt{(m - M - \sqrt{2})^2 - 1} -1}{-(m -M -\sqrt{2})} \leq \gamma_m < 1 \ \left(m \in [M], p \in (0,1] \right).
\end{align}
The noise level, i.e., the degree of smoothing, is proportional to the learning rate, which immediately leads to the optimal decay rate of the decaying learning rate. Figure \ref{fig:gamma} plots the decay rate of the existing learning rate scheduler and the regions satisfying inequality (\ref{eq:18}) when $M=200$. Figure \ref{fig:gamma-big} is an enlarged version of Figure \ref{fig:gamma}.

\begin{figure}[htbp]
\begin{minipage}[t]{0.49\columnwidth}
\centering
\includegraphics[width=\columnwidth]{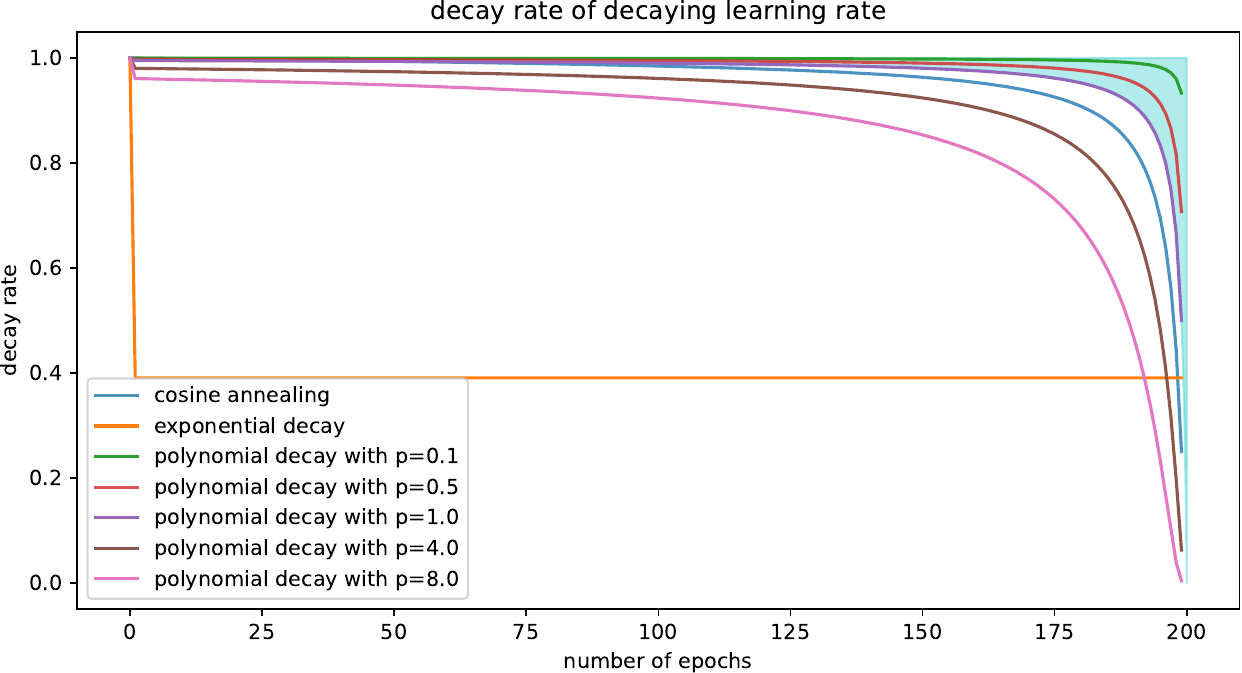}%
\caption{Decay rate of existing learning rate schedules. Blue-filled area represents regions satisfying inequality (\ref{eq:18}) when $M=200$.}
\label{fig:gamma}
\end{minipage}
\hspace{0.01\columnwidth}
\begin{minipage}[t]{0.49\columnwidth}
\centering
\includegraphics[width=\columnwidth]{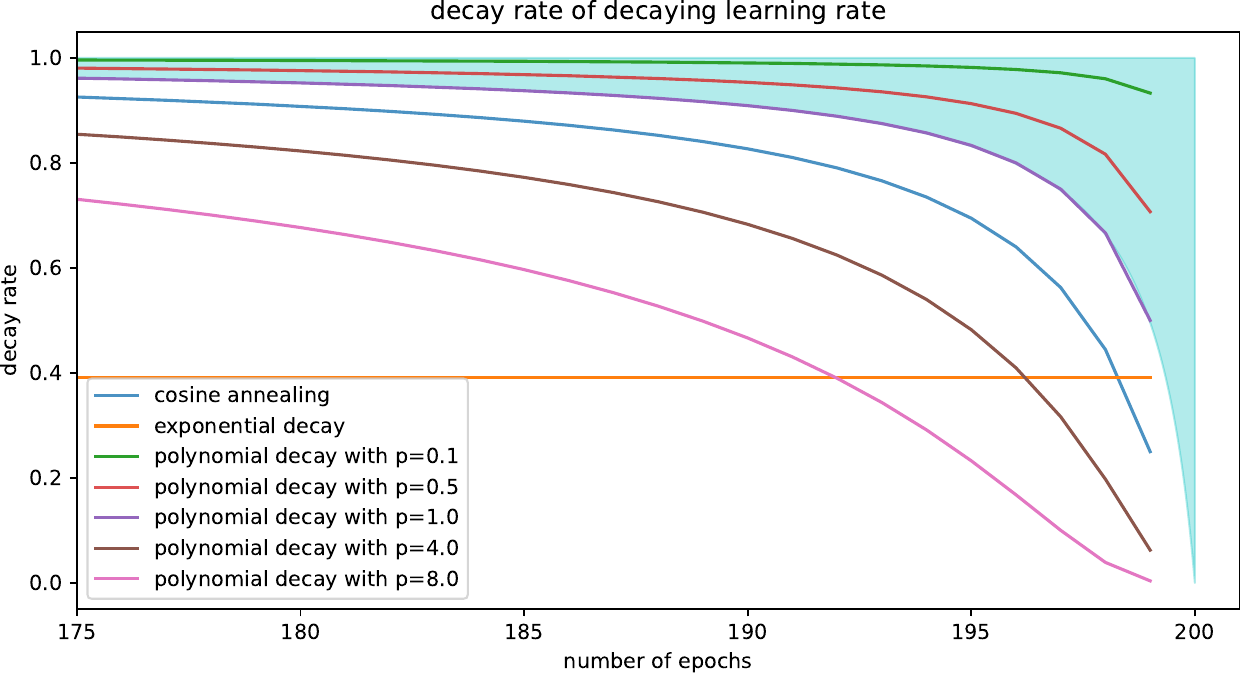}%
\caption{An enlarged version of Figure \ref{fig:gamma}.}
\label{fig:gamma-big}
\end{minipage}
\end{figure}

As shown in Figures \ref{fig:gamma} and \ref{fig:gamma-big}, inequality (\ref{eq:18}) is satisfied by polynomial decay with $p \in (0,1]$ and not by any other learning rate scheduler. The rate of polynomial decay can be written as $\frac{(M-m)^p}{\left\{ M - (m-1)\right\}^p}$, which is used in Algorithms \ref{alg:gnc2} and \ref{alg:gnc3}. Figures \ref{fig:gamma} and \ref{fig:gamma-big} also show that polynomial decay with $p \in (0,1]$ may give the smallest loss function value through implicit graduated optimization. Sato and Iiduka confirmed this experimentally for SGD; i.e., they showed that polynomial decay with $p \in (0,1]$ can achieve the smallest training loss function value for SGD, as per theory \cite{Sato2023Usi}. We thus clarified whether polynomial decay with $p \in (0,1]$ can achieve the smallest loss function value for SGD with momentum as well.

\begin{figure}[htbp]
\begin{minipage}[htbp]{0.49\columnwidth}
\centering
\includegraphics[width=\columnwidth]{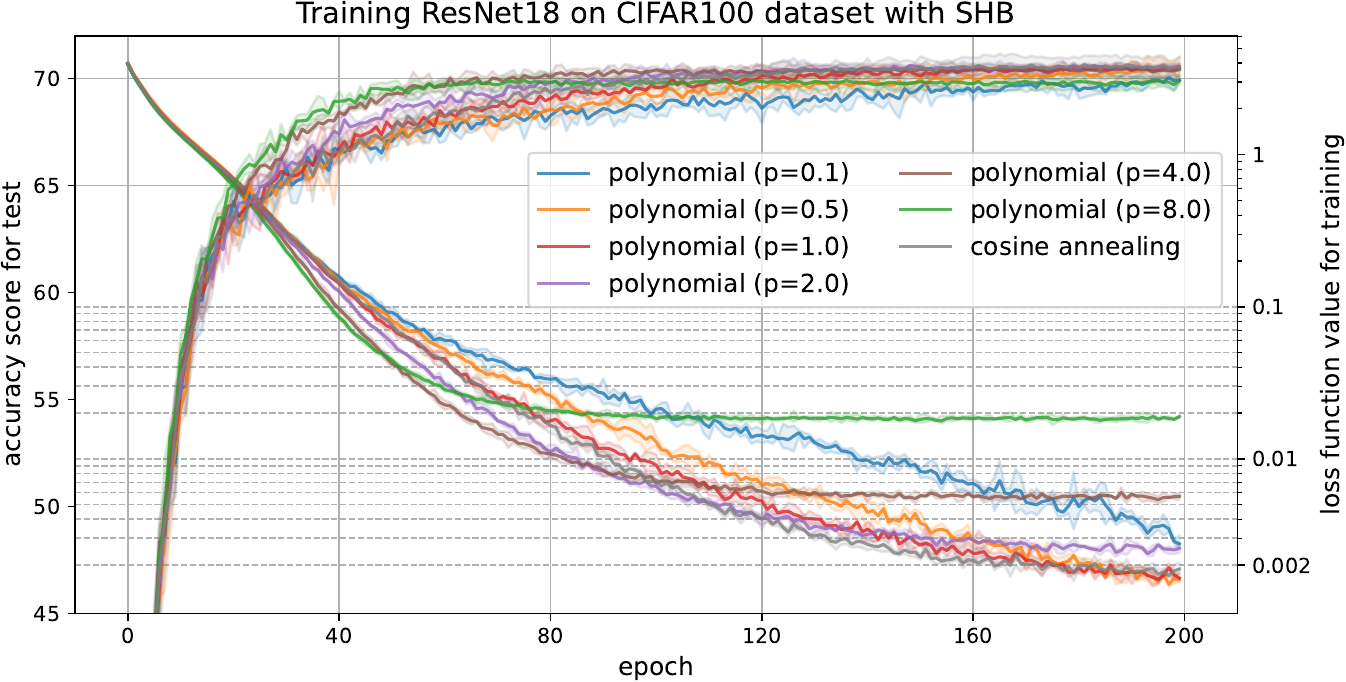}%
\caption{Accuracy score for testing and loss function value for training versus epochs in training of ResNet18 on CIFAR100 dataset with \textbf{SHB} and learning rate scheduler.}
\label{fig:13}
\end{minipage}
\hspace{0.01\columnwidth}
\begin{minipage}[htbp]{0.49\columnwidth}
\centering
\includegraphics[width=\columnwidth]{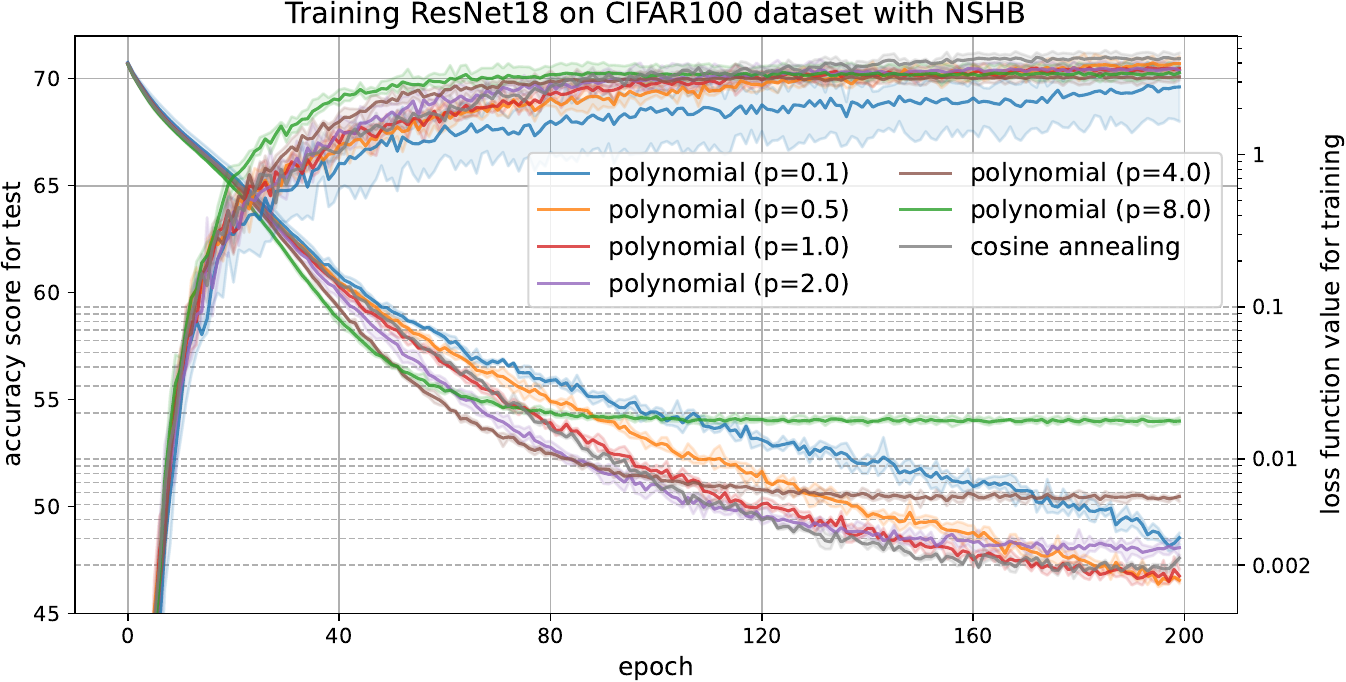}%
\caption{Accuracy score for testing and loss function value for training versus epochs in training of ResNet18 on CIFAR100 dataset with \textbf{NSHB} and learning rate scheduler.}
\label{fig:14}
\end{minipage}
\end{figure}

Figures \ref{fig:13} and \ref{fig:14} plot test accuracy and training loss function values versus batch size for SHB and NSHB with learning rate scheduler in training of ResNet18 on the CIFAR100 dataset. The initial value of the learning rate was $0.1$, and the power of the exponential decay was $0.9$. The batch size and momentum were constants, $256$ and $0.9$, respectively.

Figures \ref{fig:15} and \ref{fig:16} plot test accuracy versus training loss function values for SHB and NSHB with learning rate scheduler in training of WideResNet-28-10 on the CIFAR100 dataset.

\begin{figure}[htbp]
\begin{minipage}[htbp]{0.49\columnwidth}
\centering
\includegraphics[width=\columnwidth]{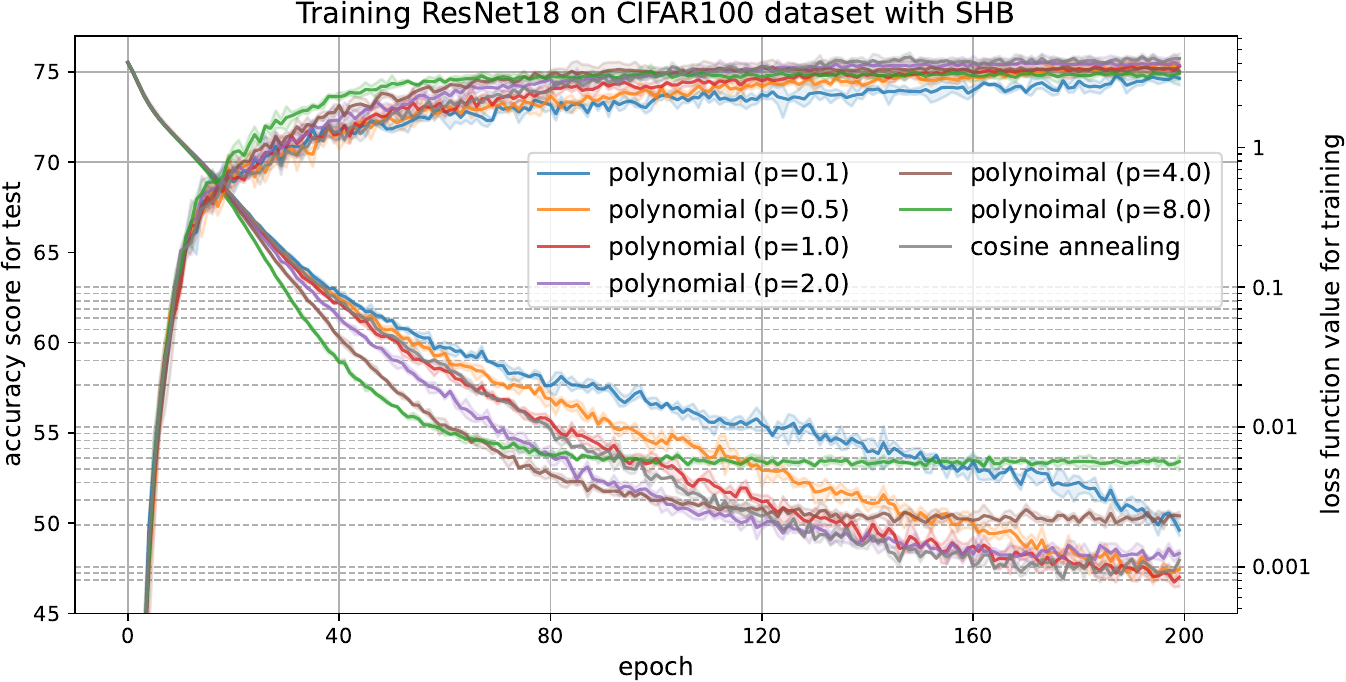}%
\caption{Accuracy score for testing and loss function value for training versus epochs in training of WideResNet-28-10 on CIFAR100 dataset with \textbf{SHB} and learning rate scheduler.}
\label{fig:15}
\end{minipage}
\hspace{0.01\columnwidth}
\begin{minipage}[htbp]{0.49\columnwidth}
\centering
\includegraphics[width=\columnwidth]{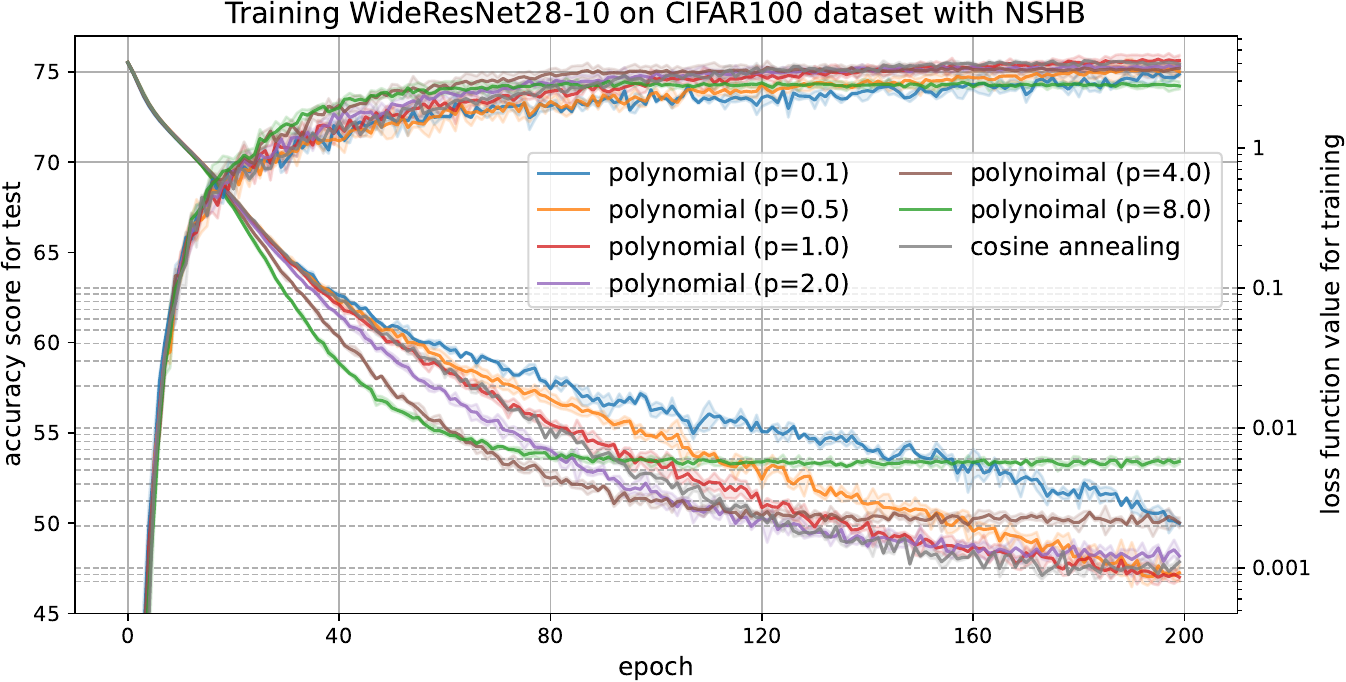}%
\caption{Accuracy score for testing and loss function value for training versus epochs in training of WideResNet-28-10 on CIFAR100 dataset with \textbf{NSHB} and learning rate scheduler.}
\label{fig:16}
\end{minipage}
\end{figure}

Figures \ref{fig:13}-\ref{fig:16} show that polynomial decay with $p \in (0,1]$ can achieve the smallest loss function values among the learning rate schedulers. Note that when $p=0.1$, the loss function value fell more slowly than when $p \in \{0.5, 0.9, 1\}$. This can be explained by the convergence rate of Algorithms \ref{alg:gnc2} and \ref{alg:gnc3} being $\mathcal{O}\left( 1/\epsilon^{\frac{1}{p}} \right) \ \left(p \in (0,1] \right)$. That is, the closer $p$ is to $1$, the fewer iterations are needed, so when $p = 0.1$, we need many more iterations than when $p \in \{ 0.5, 0.9, 1\}$.

\end{document}